%% file: JEI.tex
\documentclass[12pt]{spieman}  
\usepackage{amsmath,amsfonts,amssymb}
\usepackage{graphicx}
\usepackage{setspace}
\usepackage{tocloft}
\usepackage{array}
\usepackage{url}
\usepackage{color}
\usepackage{times}
\usepackage[utf8]{inputenc}
\usepackage{version}
\usepackage{relsize}
\usepackage{paralist}
\usepackage[super]{nth}
\usepackage{capt-of,etoolbox}
\usepackage{multirow}
\usepackage{xcolor}
\usepackage{algorithm}
\usepackage[noend]{algpseudocode}
\usepackage{rotating}
\usepackage{pifont}

\usepackage{makecell}

\newcounter{nvcptr}
\stepcounter{nvcptr}

\newenvironment{proposition}{\normalsize{\textbf{Proposition \thenvcptr.}} \stepcounter{nvcptr}} 

\newcommand{\GG}{\mathcal{G}}
\newcommand{\SSS}{\mathcal{S}}

\graphicspath{{fig/}}

\title{Evaluation Framework of Superpixel Methods with a Global Regularity Measure}

\author[a,b,c,d,*]{Rémi Giraud}
\author[a,b,e]{Vinh-Thong Ta}
\author[c,d]{Nicolas Papadakis}
\affil[a]{University of Bordeaux, LaBRI, UMR 5800, PICTURA, F-33400 Talence, France.}
\affil[b]{CNRS, LaBRI, UMR 5800, PICTURA, F-33400 Talence, France.}
\affil[c]{University of Bordeaux, IMB, UMR 5251, F-33400 Talence, France.}
\affil[d]{CNRS, IMB, UMR 5251, F-33400 Talence, France.}
\affil[e]{Bordeaux INP, LaBRI, UMR 5800, PICTURA, F-33400 Talence, France.}

\cftpagenumbersoff{figure}
\cftpagenumbersoff{table}

\begin{document} 
\maketitle

\begin{abstract}

In the superpixel literature, the comparison of state-of-the-art methods
can be biased by the non-robustness of some metrics
to decomposition aspects, such as the superpixel scale.
Moreover, most recent decomposition methods
allow to set a shape regularity parameter,
which can have a substantial impact on the measured performances.
In this paper,
we introduce an evaluation framework,
that aims to unify the comparison process of superpixel methods.
We investigate the limitations of existing metrics, 
and propose to evaluate each of the three core decomposition aspects: 
color homogeneity, respect of image objects and shape regularity.
To measure the regularity aspect, we propose a
new global regularity measure (GR), 
which addresses the non-robustness of state-of-the-art metrics.
We evaluate recent superpixel methods with these criteria, 
at several superpixel scales and regularity levels.
The proposed framework 
reduces the bias in the comparison process of state-of-the-art superpixel methods.
Finally, we demonstrate that the proposed GR measure is correlated with the performances
of various applications.
\end{abstract}


{\noindent \footnotesize\textbf{*}Rémi Giraud,  \linkable{remi.giraud@labri.fr} }

\begin{spacing}{2}   

\section{Introduction\label{sec:intro}}

Superpixel decomposition methods, that group pixels into homogeneous regions, 
have become popular with Ref. \citenum{ren2003},
and have been widely proposed in the past years 
\cite{felzenszwalb2004,vedaldi2008,moore2008,levinshtein2009,veksler2010,liu2011,achanta2012,vandenbergh2012,
conrad2013,buyssens2014,machairas2015,li2015,yao2015,rubio2016,giraud2017_scalp}.
Such decompositions provide a low-level image representation,
while trying to respect the image contours.
The superpixels are usually used as a pre-processing in 
many computer vision methods such as:
  object localization \cite{fulkerson2009},
stereo and occlusion processing \cite{zhang2011stereo},
   contour detection and segmentation \cite{arbelaez2011},  
   multi-class object segmentation
\cite{gould2008,yang2010,tighe2010,gould2014,giraud2017spm},
data associations across views \cite{sawhney2014},
   face labeling \cite{kae2013}, 
   or adapted neural networks architectures \cite{liu2015deep,gadde2016superpixel}.
For most methods, the computational cost depends on the number of elements to process,
and the superpixel representation is well adapted since this aspect is drastically reduced.
Moreover, superpixels can also directly improve the accuracy of applications such as labeling  \cite{arbelaez2011}, 
since they gather homogeneous pixels in larger areas, thus reducing the potential noise
of independent pixel labeling. 
Finally, contrary to the classical multi-resolution approach, that decreases the image size by averaging the information,
the superpixels preserve the image geometry and content.

The definition of an optimal superpixel decomposition depends on the application.
Nevertheless,
according to the literature,
the desired properties of a superpixel method should be the following.
(i) The color clustering must group pixels into homogeneous areas in terms of color features, for instance
in RGB or CIELab space.
(ii) The decomposition boundaries should adhere to the image contours, \emph{i.e.}, the
superpixels should not overlap with several image objects.
(iii) The superpixels should have regular shapes and consistent sizes within the decomposition.
Regular superpixels provide
an easier analysis of the image content, 
and reach better performances for applications such as tracking \cite{chang2013,reso2013}.
Since these properties cannot be optimal at the same time \cite{neubert2014compact}, 
most decomposition methods compute a trade-off between these aspects in their model.
For instance, Figure \ref{fig:trade_off} shows synthetic examples (a) and (d)
where a trade-off between the considered aspects  must be computed 
to decompose the images into three superpixels.
In Figures \ref{fig:trade_off}(b) and (e), the color homogeneity (i) is optimal,
while the respect of image contours (ii) and shape regularity (iii) are respectively favored in (c) and (f).
Such examples enable to illustrate how a superpixel decomposition method
should behave to optimize each criteria.

          \newcommand{\bytroish}{0.11\textwidth}
          \newcommand{\bytroisw}{0.31\textwidth}
  \begin{figure}[t!]
\begin{center}
\includegraphics[width=0.96\textwidth]{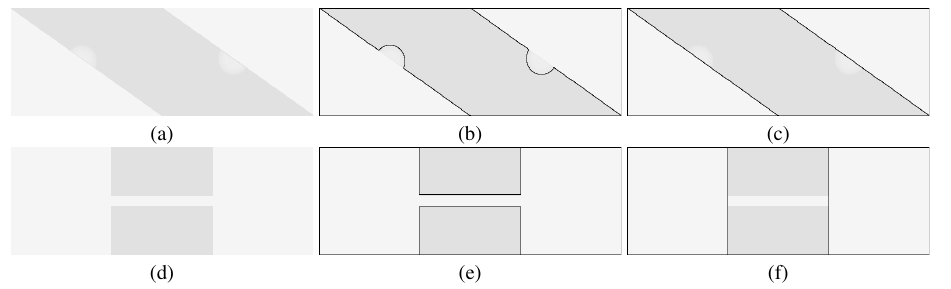}
  \end{center}
  \caption{
Examples of trade-off between the superpixel properties
for the decomposition of images (a) and (d) into three superpixels.
The decomposition in (b) and (e), are optimal in terms of color homogeneity
while the ones in (c) and (f) are respectively optimal in terms of
respect of image contours and regularity.
  } 
  \label{fig:trade_off}
  \end{figure}

Many evaluation metrics dedicated to the superpixel context were proposed to evaluate the consistency of a decomposition
with respect to the three above-mentioned superpixel properties \cite{levinshtein2009,liu2011,achanta2012,schick2012,neubert2012,wang2017superpixel}.
The metrics namely include 
intra-cluster variation (ICV) \cite{benesova2014} or explained variation (EV) \cite{moore2008}
for color homogeneity (i),
the undersegmentation error (UE) \cite{levinshtein2009,achanta2012,neubert2012}, 
achievable segmentation accuracy (ASA) \cite{liu2011}, 
or boundary recall (BR) \cite{martin2004} for the criteria (ii), on the respect of image objects,
 and circularity (C) \cite{schick2012}  or mismatch factor \cite{machairas2015}
for the criteria (iii), of superpixel shape regularity.
These metrics offer different interpretations and evaluations of the superpixel properties.
{\color{black}For instance, the circularity  \cite{schick2012} and the local regularity metric introduced in Ref. \citenum{wang2017superpixel}
respectively evaluate circular and square shapes
as the most regular ones, which can be arguable.}

Many recent methods allow the user to set a  
parameter that enforces or relaxes the regularity constraint of the superpixel shape, \emph{e.g.}, Ref.
\citenum{achanta2012,buyssens2014,machairas2015,li2015,yao2015,giraud2017_scalp}.
We show in Figure \ref{fig:slic_ex} that a same superpixel method
can produce different results, according to this parameter.
While regular decompositions provide visually consistent superpixels, with approximately the 
same shape and size, the irregular one
more accurately follows the color variations, at the expense of dissimilar superpixel shapes.

\begin{figure}[h]
 \begin{center}
\includegraphics[width=0.84\textwidth]{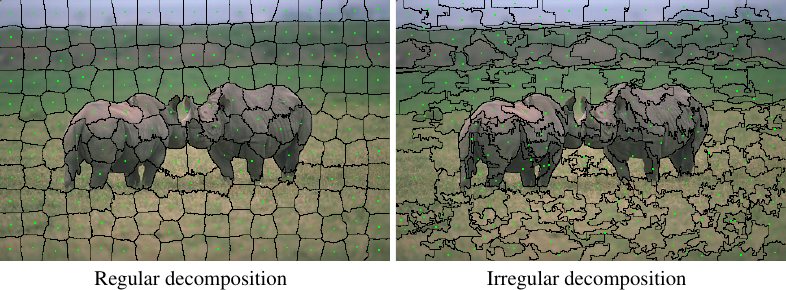}
\end{center}
\caption{
Example of decompositions
obtained with the same method \cite{achanta2012}
with different regularity settings.
The difference between the decompositions can be expressed with
the variance of distances between superpixel barycenters (normalized by the average distance),
which is $0.632$ for the irregular decomposition and 
$0.406$ for the regular one.
} 
\label{fig:slic_ex}
\end{figure}

Since the regularity can be set in most recent methods \cite{vandenbergh2012,achanta2012,buyssens2014,machairas2015,li2015,yao2015,giraud2017_scalp}, 
this parameter
should be carefully tunned during the evaluation of performances.
In each superpixel paper that introduces a decomposition method with such parameter, the authors give a default regularity 
setting that generally provides a trade-off between segmentation accuracy and superpixel shape regularity.
However, when comparing the state-of-the-art methods, 
even if the results are computed with the default settings of the initial paper,
there can be an important bias in the comparison,
since most evaluation metrics or superpixel-based pipelines achieve very different performances depending on the regularity of the decompositions.

In the recent study of Ref. \citenum{stutz2016}, no mention is made on the used settings, 
although the ranking between methods could be altered with different regularity parameters,
while Ref. \citenum{strassburg2015influence} shows that the regularity can be in line with the performances of segmentation and labeling applications.
Therefore, we recommend to evaluate the performances with accurate metrics,
and for several superpixel scales and regularity levels.
Hence, the performances of superpixel methods are not biased by the chosen 
regularity setting.

\subsection{Contributions}
In this paper, 
we propose an evaluation framework of superpixel methods.
This work aims to unify the evaluation process across superpixel works and
to enable a clear assessment of their performances. 
\begin{itemize}
\item We take a global view of existing superpixel metrics
and investigate their limitations. 
We show that the evaluation can be reduced
to one criteria per decomposition property:
color homogeneity, respect of image objects and shape regularity.
\item We address the non-robustness of the state-of-the-art regularity metrics,
with the proposed global regularity (GR) measure, that
relevantly evaluates the shape regularity and consistency of the superpixels.
\footnote{An implementation of the proposed GR metric is available at:
\url{www.labri.fr/~rgiraud/downloads}}
\item We report an evaluation of the state-of-the-art methods on this GR criteria.
Contrary to the standard superpixel literature, to reduce the bias in the evaluation process,
we recommend to perform the evaluation for several regularity levels, evaluated with GR,
since it expresses the range of potential results of each superpixel method.
\item Finally,
we demonstrate on various applications 
that the regularity constraint has a substantial impact on performances,
and that our GR metric is correlated with the obtained results.
\end{itemize}

\subsection{Outline}
In this paper, we first present in Section \ref{sec:metrics} 
the existing superpixel metrics, and investigate their limitations.
To address the non-robustness of the regularity metrics, we
propose in Section \ref{sec:gr}
a new global regularity measure.
In Section \ref{sec:soa}, we use the considered metrics to
compare the state-of-the-art superpixel methods 
according to several superpixel scales and regularity levels.
Finally, we demonstrate the impact of the regularity setting
on several applications in Section \ref{sec:impact_regu}.

\section{Standard Superpixel Metrics and Limitations\label{sec:metrics}}

In this section, we present 
the existing superpixel metrics,
that were progressively introduced in the literature
to evaluate the color homogeneity, the respect of image objects or the shape regularity,
and we show their limitations 
to only focus on relevant ones.
In Sections  \ref{sec:color} and \ref{sec:regu},
we present color homogeneity and shape regularity  metrics that evaluate
for an image $I$, 
a superpixel decomposition 
$\SSS=\{S_k\}_{k\in \{1,\dots,|\SSS|\}}$
composed of $|\SSS|$ superpixels $S_k$,
where $|.|$ denotes the cardinality of the considered element.
In Section \ref{sec:rio}, 
we present metrics that evaluate the respect of image objects and
compare the decomposition to a ground truth denoted 
$\GG=\{G_j\}_{j\in \{1,\dots,|\GG|\}}$, with $G_j$ a segmented region.

\subsection{Homogeneity of Color Clustering\label{sec:color}}

The homogeneity of the color clustering is a core aspect of the superpixel decomposition.
Most methods
compute a trade-off between spatial and color distances to compute the superpixels.
The ability to gather homogeneous pixels should hence be considered in the comparison process,
but this aspect is rarely evaluated in state-of-the-art superpixel works.
In Section \ref{sec:impact_regu}, we show that color homogeneity is
particularly interesting for image compression \cite{levinshtein2009,wang2013}.

The homogeneity of colors can be evaluated by comparing the average colors of superpixels 
to the colors of the pixels in the initial image.
The intra-cluster variation (ICV) \cite{benesova2014} has been proposed to measure 
such color variation for a decomposition $\SSS$:
\begin{align}
 \text{ICV}(\SSS)  &= \frac{1}{|\SSS|}\sum_{S_k}\frac{\sqrt{\sum_{p\in S_k}{(I(p)-\mu(S_k))^2}}} 
              {|S_k|} , \label{ev}
\end{align}
with $\mu(S_k)$, the average color of the superpixel $S_k$.

The explained variation (EV) \cite{moore2008}
was also proposed to evaluate
the homogeneity of the color clustering and is defined as:
\begin{align}
 \text{EV}(\SSS)   
 &= \frac{\sum_{S_k}{|S_k|\left(\mu(S_k) - \mu(I)\right)^2}} 
              {\sum_{p\in I}{\left(I(p)-\mu(I)\right)^2}} . \label{ev}
\end{align}

\subsubsection*{Limitations}

The ICV metric presents several drawbacks.
It is not normalized by the image variance to allow comparable evaluation
between different images \cite{stutz2016}, 
and the superpixel size is not considered, so the measure is not robust to the superpixel scale.
In Figure \ref{fig:icv}, we illustrate these issues on synthetic examples.
The dynamic and size of the same image with the same decomposition are altered, and
these transformations impact the measure of ICV.

  \begin{figure}[h!]
\begin{center}
\includegraphics[width=0.92\textwidth]{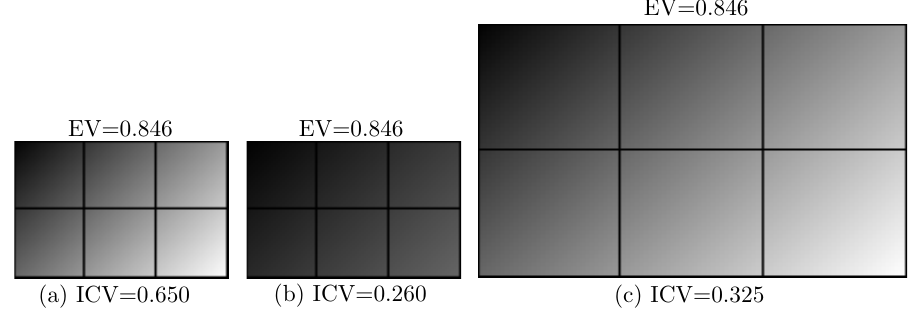}
  \end{center}
  \caption{
  Comparison of ICV and EV to measure the superpixel color homogeneity.
  The  dynamic and dimension of the same image (a) are respectively modified in (b) and (c).
  EV is robust to these transformations contrary to ICV.
  } 
  \label{fig:icv}
  \end{figure}

  The EV metric is robust to these transformations.
 It expresses the color variance contained 
into each superpixel, and simple calculations 
provide another explicit formulation:
\begin{align}
 \text{EV}(\SSS)  &= 1 - \sum_{S_k}\frac{|S_k|}{|I|}.\frac{\sigma(S_k)^2}{\sigma(I)^2} .
\end{align}

In Figure \ref{fig:ev}(c), we illustrate the color variance $\sigma(S_k)^2$ within each  
superpixel $S_k$.
Higher variance within a superpixel will lower the EV criteria, so high EV values
express homogeneous color clustering.
Despite the robustness of  this criteria, 
it was not considered in the main state-of-the-art superpixel works.
As in Ref. \citenum{stutz2016}, we recommend the use of EV to evaluate color homogeneity.

          \newcommand{\htabm}{0.22\textwidth}
        \newcommand{\htabmm}{0.24\textwidth}
        \newcommand{\wtabm}{0.30\textwidth}
        \newcommand{\wtabz}{0.28\textwidth}
        \newcommand{\htabz}{0.24\textwidth}

\begin{figure}[t!]
\begin{center}
{\footnotesize
 \begin{tabular}{@{\hspace{0mm}}c@{\hspace{1mm}}c@{\hspace{1mm}}c@{\hspace{0mm}}}
\includegraphics[width=0.96\textwidth]{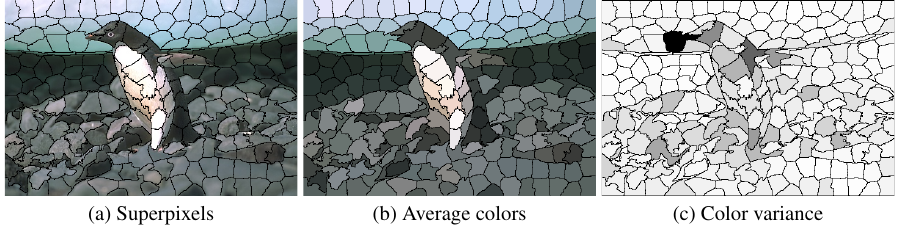}
\end{tabular}
  }
  \end{center}
  \caption{  
Example of superpixel decomposition (a), with average colors (b), and
color variance $\sigma (S_k)^2$ within each superpixel $S_k$ (c).
  }
  \label{fig:ev}
  \end{figure}

\subsection{Respect of Image Objects\label{sec:rio}}

The most considered aspect
in the superpixel evaluation is the respect of the image objects.
An accurate decomposition 
should have superpixels that do not overlap with multiple objects.
To evaluate this aspect, many metrics were proposed.
In the literature, undersegmentation error (UE),
achievable segmentation accuracy (ASA)
and boundary recall (BR) measures are mainly reported.
While UE and ASA compute overlaps with the ground truth regions, 
BR is a contour adherence measure, which is usually correlated to 
the first two metrics.

Regarding the UE, several definitions were proposed \cite{levinshtein2009,achanta2012,neubert2012}.
UE evaluates the number of pixels that cross ground truth region boundaries.
The initial UE formulation \cite{levinshtein2009}, denoted $\text{UE}_{\text{L}}$,
is defined as:
\begin{equation}
\text{UE}_\text{L}(\SSS,\GG) = \frac{1}{|\GG|}\sum_{G_j}\frac{\left(\sum_{S_k,S_k\cap G_j\neq\emptyset}|S_k|\right) - |G_j|}{{|G_j|}}.  \label{ue_l}
\end{equation}
The $\text{UE}_\text{L}$ measure was discussed in several works, \emph{e.g.}, Ref. \citenum{achanta2012,vandenbergh2012,neubert2012},
since 
any superpixel that has an overlap with the ground truth segment penalizes the metric.
Therefore, $\text{UE}_\text{L}$ is very sensitive to small overlaps and does not accurately reflect the 
respect of the image objects.
A method to reduce the overlap considered  in \eqref{ue_l} was proposed in Ref. \citenum{achanta2012}, but
it requires a parameter.
Recent state-of-the-art works tend to use the free parameter formulation of UE introduced by Ref. \citenum{neubert2012}:
\begin{equation}
\text{UE}(\SSS,\GG) = \frac{1}{|I|}\sum_{S_k}\sum_{G_j}{\min\{|S_k\cap G_j|,|S_k\backslash G_j|\}}.  \label{ue_np}
\end{equation}
This formulation respectively considers the intersection or the non-overlapping superpixel area
in case of small or large overlap with the ground truth region, 
and addresses the non-robustness of former UE definitions.

The ASA measure \cite{liu2011} also
aims at evaluating the overlap of superpixels with the ground truth.
It is reported in most of the superpixel literature as follows:
\begin{equation}
\text{ASA}(\SSS,\GG) = \frac{1}{|I|}\sum_{S_k}\underset{G_j}{\max}|S_k\cap G_j|.  \label{asa}
\end{equation}
For each superpixel $S_k$, 
the largest possible overlap with a ground truth region $G_j$ is considered,
and higher values of ASA indicate better results.

Another measure is extensively reported in the superpixel literature to evaluate 
the adherence to object contours: the boundary recall (BR).
This metric evaluates the detection of ground truth contours $\mathcal{B(\GG)}$ 
by the superpixel boundaries $\mathcal{B}(\SSS)$ such that:
\begin{equation}
\text{BR}(\SSS,\GG) = \frac{1}{|\mathcal{B}(\GG)|}\sum_{p\in\mathcal{B}(\GG)}\delta[\min_{q\in\mathcal{B}(\SSS)}\|p-q\|< \epsilon]  ,   \label{br}
\end{equation}
\noindent with 
$\delta[a]=1$ when $a$ is true and $0$ otherwise,
and $\epsilon$ is a distance threshold that has to be set, 
for instance to $2$ pixels \cite{liu2011,vandenbergh2012}. 
Each ground truth pixel is considered as detected, if
a superpixel boundary is at less than an $\epsilon$ distance. \\

  \subsubsection*{Limitations}
  Although the BR measure \eqref{br} has been extensively reported in the literature
and is recommended in Ref. \citenum{stutz2016},
it does not express the respect of image objects or the contour detection performances, 
as it only considers true positive contour detection.
Hence, the number of computed superpixel contours is not considered, and 
very irregular methods can obtain higher BR results.
Figure \ref{fig:br} compares two decompositions with the same number of superpixels 
that have maximal $\text{BR}=1$.
One of the decomposition is very irregular and produces a high number of boundary superpixels,
and this aspect is not considered in the BR metric.
For these reasons, recent works such as Ref. \citenum{machairas2015,zhang2016,giraud2017_scalp} 
report BR results according to the contour density (CD),
which measures the number of superpixel boundaries such that 
$\text{CD}(\SSS)=|\mathcal{B}(\SSS)|/|I|$.
Nevertheless, BR needs a parameter $\epsilon$ to be set and 
is not sufficient to reflect the respect of image objects.
BR should only be considered as a tool for evaluation of contour detection performances,
as shown in Section \ref{sec:impact_regu}. \medskip

\newcommand{\brcaw}{0.22\textwidth}
\newcommand{\brcah}{0.14\textwidth}
\begin{figure}[h!]
\begin{center}
{\scriptsize
 \begin{tabular}{@{\hspace{0mm}}c@{\hspace{1mm}}c@{\hspace{1mm}}c@{\hspace{1mm}}@{\hspace{0mm}}}
\includegraphics[width=0.96\textwidth]{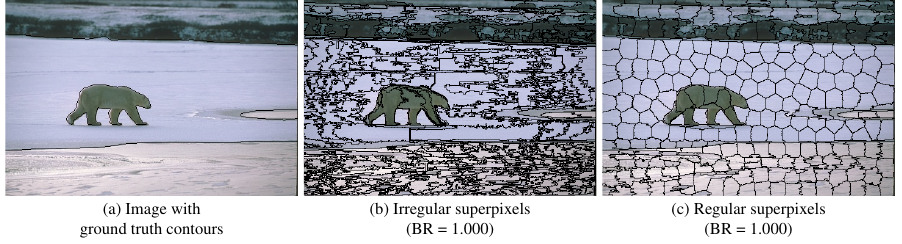}
\end{tabular}
  }
  \end{center}
  \caption{
  Examples of irregular \cite{yao2015}  (b) and regular \cite{achanta2012} (c) decomposition of an image (a),
  with maximal BR measure.
  BR evaluates the detection of ground truth contours so both decompositions 
  can have maximal BR measure,
  although the irregular one produces more superpixel boundaries.
  } 
  \label{fig:br}
  \end{figure}

The study of Ref. \citenum{stutz2016} claims that the 
UE \eqref{ue_np} and ASA \eqref{asa} measures are correlated, so that $\text{ASA}=1-\text{UE}$,
making redundant the use of both for superpixel evaluation.
Under relevant assumptions, we now show that the true relation between the two measures is:
\begin{equation}
 \text{ASA}(\SSS,\GG)=1-\text{UE}(\SSS,\GG)/2. \label{asa_ue}
\end{equation}
In Appendix \ref{sec:asa_demo}, we demonstrate that the relation
\eqref{asa_ue} is exact when all superpixels have a major overlap, \emph{i.e.}, 
a ground truth region intersects with more than half of the superpixel area.
This case is illustrated in Figure \ref{fig:cond_ue_asa}(b), where
a superpixel $S_k$ has a major overlap with a region $G_2$.
Note that this assumption is necessarily true when each superpixel overlaps with only two regions
or when the ground truth is binary.
In Appendix \ref{sec:asa_demo}, we measure the error of \eqref{asa_ue}
on state-of-the-art methods.
This error appears to be negligible, underlying the likelihood of this assumption.

  \begin{figure}[t!]
\begin{center}
\includegraphics[width=0.90\textwidth]{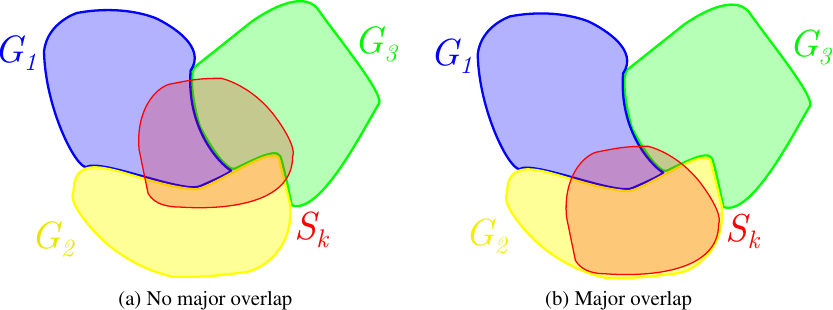}
\end{center}
\caption{
Examples of decomposition where a superpixel $S_k$ overlaps with multiple ground truth regions $G_j$.
In (b), contrary to (a), a region ($G_2$) overlaps with more than half of $S_k$,
which corresponds to the major overlap hypothesis ensuring \eqref{asa_ue}.
} 
\label{fig:cond_ue_asa}
\end{figure}

Hence, since the UE$_\text{L}$ measure has been proven to be non-robust,
and we demonstrate that the study of UE and ASA is equivalent,
we recommend to restrain the evaluation of the respect of image objects
to the ASA results \eqref{asa}. \\

\subsection{Regularity\label{sec:regu}}

Superpixel decompositions compute a lower-level image representation.
Since the desired number of elements is given, 
and most methods tend to compute regions with approximately the same size,
superpixels are generally created in smooth areas (see for instance Figure \ref{fig:slic_ex}).
To produce a consistent decomposition,
most superpixel methods compute a trade-off between color clustering and regularity,
which has been proven to
have an impact on application performances \cite{veksler2010,reso2013,strassburg2015influence}.
Therefore, the regularity of the superpixel shapes is a core aspect when evaluating
and comparing decomposition methods.

In Ref. \citenum{schick2012}, the circularity metric was introduced
to locally evaluate the compactness of the superpixels.
This measure is the usual local regularity metric, and has been considered in state-of-the-art works 
\cite{reso2013,buyssens2014,zhang2016,giraud2017_scalp}, and 
benchmarks \cite{schick2014,stutz2016,wang2017superpixel}.
This circularity C is defined 
 for a superpixel shape $S$ as follows:
\begin{equation}
\text{C}(S) = \frac{4\pi |S|}{|P(S)|^2} ,  \label{circu}
\end{equation}
where $P(S)$ is the superpixel perimeter.
The regularity is hence considered as the ability to produce circular areas.

{\color{black}
 While the circularity independently evaluates the local compactness of each superpixel, 
 other works propose to
 evaluate the consistency of the shapes within the decomposition.
 In Ref. \citenum{wang2017superpixel}, the variance of superpixel sizes is considered.}
 Ref. \citenum{machairas2015} goes further and proposes 
 an adapted version of the mismatch factor \cite{strachan1990} to measure the consistency in terms of size and shape.
  The mismatch factor is computed as $1-J(S_1,S_2)$ with $J$ being the standard Jaccard index \cite{jaccard1901},
  that computes the overlap between two regions $S_1$ and $S_2$ such that:
$J(S_1,S_2) = |S_1\cap S_2|/|S_1 \cup S_2|$.
 The global measure J for a superpixel decomposition $\SSS$ 
 is computed as follows \cite{machairas2015}:
     \begin{equation}
     \text{J}(\SSS) = \frac{1}{|\SSS|}\sum_{S_k} J(S_k^*,\hat{S}^*), \label{jaccard}
\end{equation}
  with $S_k^*$  the registered superpixel $S_k$, 
  so its barycenter is at the origin of the coordinates system, and
   $\hat{S}^*$, the binary average shape of $\SSS$.
In the following, we consider the J metric when comparing to Ref. \citenum{machairas2015},
since, contrary to the mismatch factor,
its interpretation is consistent with the other presented metrics, \emph{i.e.}, higher value is better.
Eq.  \eqref{jaccard} compares each superpixel
  to the binary average shape $\hat{S}^*$ of the decomposition using 
  $J$.
  To compute $\hat{S}^*$,
  the registered superpixels $S_k^*$ are first averaged into $S^*$ such that:
\begin{equation}
S^*=\frac{1}{|\SSS|}\sum_{S_k}S_k^* .
\end{equation}
  A thresholded shape $S_t^*$ is then defined by binarization 
  with respect to a threshold $t$.  
The binary average shape is finally defined as 
$\hat{S}^* = S_{\underset{t}{\text{argmax}}{(|S_t^*|\geq \mu)}}^*$,
with $\mu=|I|/|\SSS|$, 
the average superpixel size.
We illustrate these definitions in Figure \ref{fig:mf_def}.
We represent a superpixel decomposition $\SSS$,  the corresponding average of superpixel shapes $S^*$,
and the binary average shape $\hat{S}^*$.  \\

          \begin{figure}[t!]
\begin{center}
 \includegraphics[width=0.96\textwidth]{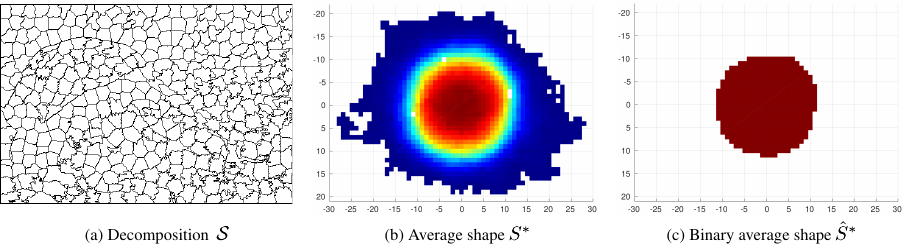}
 \end{center}
  \caption{
  Illustration of the average superpixel shape definition \cite{machairas2015}.
  A decomposition is considered in (a). The superpixel shapes are registered
  and averaged into $S^*$ (b) to provide the binary average shape (c).
  } 
  \label{fig:mf_def}
  \end{figure}

\subsubsection*{Limitations}

Although the circularity metric \eqref{circu}, introduced in Ref. \citenum{schick2012},
is the local regularity metric usually used in the literature 
\cite{reso2013,buyssens2014,zhang2016,giraud2017_scalp,schick2014,stutz2016}, 
it presents several drawbacks.
For instance, the relevance of this metric, and the regularity definition have been discussed \cite{machairas2015}, 
since the circularity considers a circle or a hexagon as more regular than a square,
and is very sensitive to boundary smoothness \cite{giraud2017_src}.
As a consequence, some methods, such as Ref. \citenum{zhang2016}, 
start from a hexagonal repartition of superpixels seeds, and design their
spatial constraint to fit to a hexagonal grid in order to obtain higher circularity.
However, the superpixel literature more generally refers to the regularity
as the ability to produce convex shapes with smooth boundaries.
{\color{black}Therefore, squares, circles or hexagons should be considered as regular shapes.}
Moreover, since most superpixel methods 
generate a regular square grid with their regularity parameter set to maximum value,
it would make sense to obtain the highest regularity measure for such square decomposition.

The mismatch factor or Jaccard index \eqref{jaccard} was introduced  by Ref. \citenum{machairas2015}
to evaluate the shape consistency within the whole decomposition, but it does not consider the superpixel size, 
so each shape equally contributes.
Moreover, by computing a thresholded average shape, the metric
appears to be non-robust to large shape outliers,
leading to potential irrelevant measures.
Finally, with such metric that only considers the shape consistency,
decomposing an image with its lines or with stretched rectangles would give the
highest regularity measure.

Contrary to ASA and EV, that give a relevant measure of aspects (i) color homogeneity and 
(ii) respect of image objects, no existing metric
provides a robust and accurate regularity measure of a superpixel decomposition.
As a consequence, we propose in the following Section \ref{sec:gr}, 
a new global regularity measure,
that addresses the limitations of state-of-the-art ones.

\section{A new Regularity Measure\label{sec:gr}}

In the literature, the circularity \eqref{circu} has been proposed
to measure the shape regularity \cite{schick2012},
and the mismatch factor to evaluate the shape consistency across the decomposition \cite{machairas2015}.
These measures present several drawbacks, that we address 
in this work by introducing 
two new metrics \eqref{src} and \eqref{smf},
combined into the proposed global regularity (GR) measure.
GR robustly evaluates both shape regularity and 
consistency over the decomposition, and
we demonstrate in Sections \ref{sec:soa} and \ref{sec:impact_regu} 
that it relevantly evaluates the performances of superpixel methods. \medskip

\subsection{Shape Regularity\label{sec:local_regu}}

As stated in Section \ref{sec:regu},
an accurate superpixel shape regularity measure
should provide the highest results for convex shapes,
such as squares, circles or hexagons,
and penalize unbalanced shapes while considering noisy boundaries.
To express such a measure,
we propose to combine all these aspects
into a new shape regularity criteria (SRC) \cite{giraud2017_src}.
The convexity of a shape $S$ is considered, \emph{i.e.},
the smoothness of its contours and the overlap 
with its convex hull $H_S$, that entirely contains $S$.
To evaluate both overlap with the convex hull and contour smoothness,
we first define as  $\text{CC}(S)=|P(S)|/|S|$, the ratio between the perimeter
and the area of a shape $S$, which is linked with the Cheeger constant for convex shapes \cite{caselles2009}.
Then, we introduce our criteria of regularity (CR) as:
\begin{equation}
\text{CR}(S)= \frac{\text{CC}(H_s)}{\text{CC}(S)}.
\end{equation}
Since the convex hull $H_s$ entirely contains $S$ and has a lower perimeter,
the CR measure is between 0 and 1,
and is maximal for convex shapes such as squares, circles or hexagons.
Nevertheless, the comparison to the convex hull is not sufficient to define the regularity of a superpixel.

The balanced repartition of the pixels within the shape is another aspect to consider.
Otherwise, shapes such as ellipses or lines would get the maximum CR results.
We define the variance term $\text{V}_{\text{xy}}$ as the ratio between  
the minimum and maximum variance of pixel positions $x$ and $y$ that belong to $S$:
\begin{equation}
\text{$\text{V}_{\text{xy}}$}(S) =  {\frac{\min(\sigma_x,\sigma_y)}{\max(\sigma_x,\sigma_y)}}   ,
  \end{equation}
\noindent with $\sigma_x$ and $\sigma_y$, the standard deviations of
the pixel positions.
Such measure enables to penalize dissimilarity in the pixel repartition, and
$\text{$\text{V}_{\text{xy}}$} = 1$ if, and only if, $\sigma_x=\sigma_y$, \emph{i.e.},
if the spatial pixel repartition is balanced.

The proposed shape regularity criteria (SRC)
is defined as follows:
\begin{equation}
\text{SRC}(\SSS) = \sum_{S_k}{\frac{|S_k|}{|I|}\text{CR}(S_k)}\text{V}_{\text{xy}}(S_k) . \label{src}
\end{equation}
Note that in practice, we use the square root of $\text{V}_{\text{xy}}$,
so both criteria have similar variation ranges.
SRC robustly and jointly evaluates
convexity, contour smoothness and balanced pixel repartition. \smallskip

\subsubsection*{Circularity vs Shape Regularity Criteria}
To demonstrate the robustness and relevance of SRC \eqref{src} over circularity  C \eqref{circu}, 
we consider in Figure \ref{fig:shape} synthetic shapes 
that are split into three groups (regular, standard and irregular), 
and generated with smooth (top) and noisy boundaries (bottom).
First, we present the circularity drawbacks,
which reports lower measure for the {Square} than for the {Hexagon}, 
or the {Ellipse}.
Since methods such as Ref. \citenum{machairas2015,zhang2016} produce superpixels from a hexagonal grid,
the regularity evaluation is very likely to be superior for these methods
than for other ones starting from square regular grids.
The circularity is also very sensitive to the contour smoothness, 
so regular and standard noisy shapes have similar measure, and 
the groups are no longer differentiated (see the bottom part of Figure \ref{fig:shape}).
Finally, standard but smooth shapes 
have much higher circularity than noisy regular ones.

With SRC, we first note that 
the three regular shapes have the highest regularity measure ($\approx1$),  
and that regular but noisy shapes,
have similar SRC to the smooth standard ones.
Overall, since SRC is less dependent on the boundary smoothness,
in each  group, smooth and noisy shapes are clearly separated, contrary to C.
Finally, we show the metric evolution with the shape size in Figure \ref{fig:c_vs_src}.
As stated in Ref. \citenum{roussillon2010}, due to discretization issues, 
the circularity must be thresholded so it is not superior to 1, and
it drops as the shape size increases.
Therefore, this metric is not robust to the superpixel size, 
and the comparison of methods on circularity is 
relevant only if the compared superpixel decompositions have the same number of elements.
Contrary to the circularity, the SRC metric
is robust to the superpixel scale and provides a consistent evaluation of shape regularity. \smallskip

\newcommand{\ec}{3mm}
\newcommand{\ecc}{2mm}
\newcommand{\shapew}{0.08\textwidth}

\begin{figure}
\begin{center}
  \includegraphics[width=.98\textwidth]{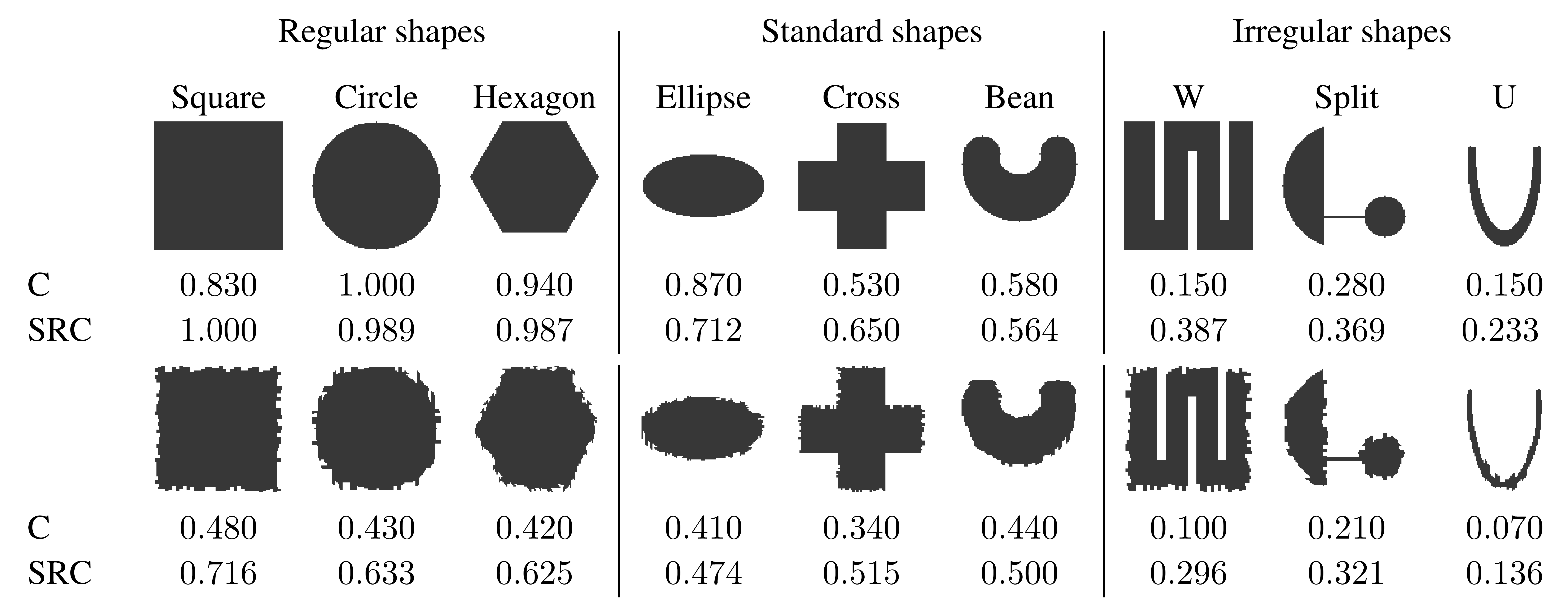}
\end{center}
\caption{
Comparison of circularity (C) and proposed shape regularity criteria (SRC)
on three groups of synthetic shapes with smooth (top) and noisy boundaries (bottom).
C appears to only favor circular shapes and
does not separate standard and regular noisy ones.
The SRC metric
addresses these issues and clearly differentiates the shape groups 
in the smooth and noisy cases.
} 
\label{fig:shape} 
\end{figure}

\begin{figure}[t!]
\begin{center}
{\footnotesize
\begin{tabular}{@{\hspace{0mm}}c@{\hspace{1mm}}c@{\hspace{0mm}}}
  \includegraphics[width=.61\textwidth,height=0.36\textwidth]{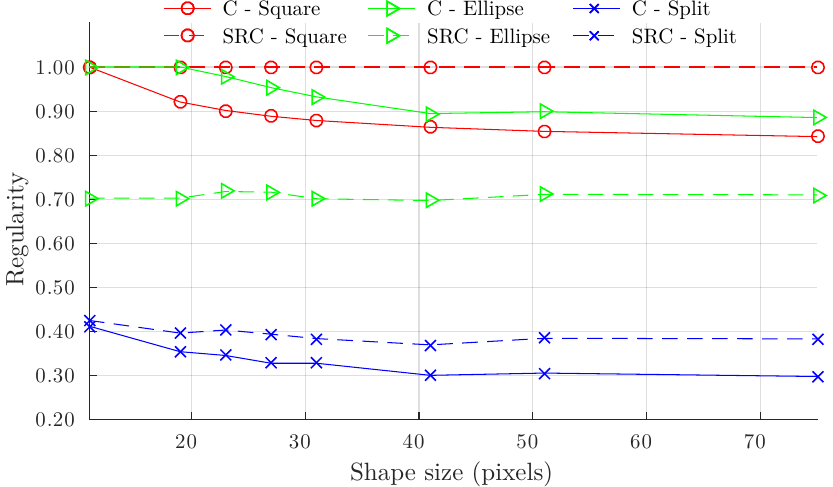}
\end{tabular}
 }
\end{center}
\caption{
Robustness to the superpixel scale of the proposed shape regularity criteria (SRC)
compared to the circularity (C).
} \vspace{-0cm}
\label{fig:c_vs_src}
\end{figure}

  Nevertheless, the SRC and circularity metrics
  independently evaluate each superpixel shape without considering 
  the global homogeneity of the decomposition.
  Superpixel methods, \emph{e.g.}, Ref. \citenum{felzenszwalb2004,rubio2016}, or other 
  segmentation algorithms such as the quadtree partitioning method \cite{tanimoto1975}
  can produce regions of very variable sizes.
  In Figure \ref{fig:quad}, 
  we show an example of decomposition with superpixels having approximately the same size \cite{achanta2012},
  and a standard quadtree-based partition \cite{tanimoto1975}, which produces larger regions in areas with lower color variance.
  Since the circularity and SRC measures
  independently consider each superpixel and report an average evaluation of local regularity, 
  the quadtree partition 
  obtains the highest measure although its elements do not have similar sizes.
  Such local metrics are thus not sufficient to express the global regularity of a decomposition.

\newcommand{\zzz}{0.23\textwidth}
\begin{figure}[h!]
\begin{center}
{\small
\begin{tabular}{@{\hspace{0mm}}c@{\hspace{0.5mm}}c@{\hspace{0mm}}}
  \includegraphics[width=.96\textwidth]{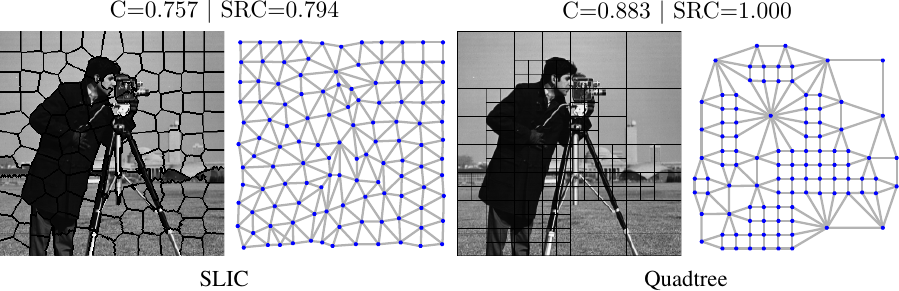}
\end{tabular}
  }
\end{center}
\caption{Limitation of the local shape regularity metrics.
The shape consistency is not considered so the quadtree decomposition 
gives higher measures than the decomposition obtained with the SLIC method \cite{achanta2012}.
The decompositions are represented with their Delaunay graphs, 
connecting the barycenters of adjacent superpixels.
} 
\label{fig:quad}
\end{figure}

\subsection{Shape Consistency\label{subsubsection:regu_smf}}

  The regularity of the superpixel shapes evaluated with a local criteria
  is a relevant information but does not reflect their
  consistency in terms of shape and size within the decomposition.
  In Section \ref{sec:local_regu}, we define the shape regularity properties for a superpixel:
  convexity, boundary smoothness and balanced pixel repartition. 
  Nevertheless, a perfectly regular decomposition of an image should be composed of similar regular shapes at the same scale.
  In other words, a relevant criteria should locally evaluate the
  shape regularity and how the superpixels are consistent in terms of shape and size within the decomposition.
As stated in Section \ref{sec:regu}, 
the mismatch factor \eqref{jaccard} uses the Jaccard index to compare the superpixels to a
thresholded average shape.
As illustrated in Figure \ref{fig:mf_def2},
this measure can be incoherent with the visual consistency of a decomposition.
In these examples, the binary average shape $\hat{S}^*$ is the same, and corresponds to the red areas. 
The J measure is low and incoherently decreases as the consistency is visually improved.

In this work, we propose the smooth matching factor (SMF),
that directly compares the superpixels to the
average shape $S^*$:
\begin{equation}
\text{SMF}(\SSS) = 1- \sum_{S_k}\frac{|S_k|}{|I|}.\left\|\frac{S^*}{|S^*|} - \frac{S_k^*}{|S_k^*|}\right\|_1/{2}. \label{smf}
\end{equation}
SMF compares the spatial distributions of the average superpixel shape $S^*$ 
to each registered superpixel shape $S_k^*$. 
The SMF criteria should be close to 1 if the distributions of pixels within the shapes are similar,
 and close to $0$ otherwise.
Overall, the proposed SMF metric evaluates the consistency in terms of shape and size within a decomposition. 
By considering the average shape $S^*$ without thresholding,
the evaluation is more robust to shape outliers.
This criteria
addresses the non-robustness of the mismatch factor (J) \cite{machairas2015},
as can be seen in Figure \ref{fig:mf_def2}, 
where SMF is relevant according to the consistency of the decomposition.
Finally, 
examples of superpixel decomposition on natural images are given in Figure \ref{fig:mf_ex} and
illustrate that SMF is not correlated to the J measure.

\smallskip

          \begin{figure}[t!]
\begin{center}
  \includegraphics[width=.96\textwidth]{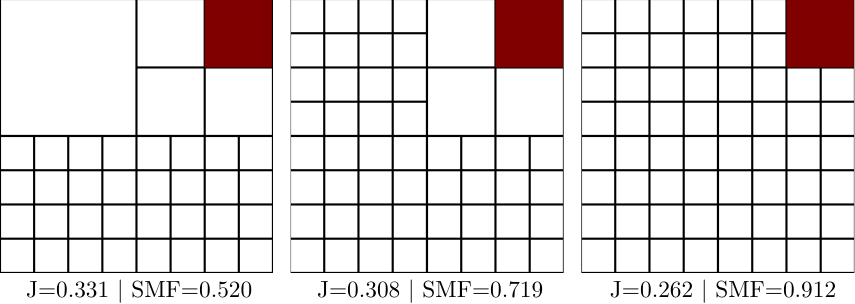}
\end{center}
  \caption{
  Illustration of several decomposition examples 
  with the corresponding average binary shape in red, and J \eqref{jaccard} and SMF \eqref{smf} values.
  } 
  \label{fig:mf_def2}
  \end{figure}

  \begin{figure}[t!]
\begin{center}
  \includegraphics[width=.96\textwidth]{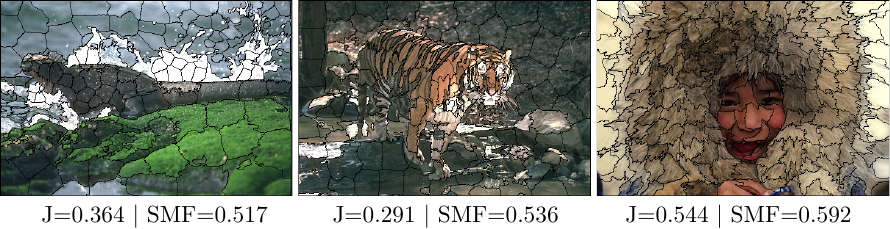}
\end{center}
  \caption{
 Decomposition examples with corresponding J  \eqref{jaccard} and SMF  \eqref{smf} measures.
  } \
  \label{fig:mf_ex}
  \end{figure}

\subsection{Global Regularity Measure\label{gr_def}}
As previously stated, 
  a perfectly regular superpixel decomposition should be composed of  
  compact shapes, that would be consistent in terms of size and shape.
  The SRC metric enables to locally evaluate the shape regularity while
  the SMF measures the shape  consistency.
To evaluate both aspects, we propose to combine these metrics in the
global regularity (GR) measure: 
\begin{equation}
\text{GR}(\SSS) = \text{SRC}(\SSS)\text{SMF}(\SSS) . \label{gr}
\end{equation}

 In the following Section \ref{sec:soa}, 
 we report the evaluation on the considered metrics (EV, ASA, GR) for state-of-the-art superpixel methods.
 These measures are reported according to the number of generated superpixels, and also 
 according to the regularity, since
 this setting substantially impacts performances.

\section{Comparison of State-of-the-Art Superpixel Methods\label{sec:soa}}

\subsection{Dataset}

We compare the performances of state-of-the-art methods  on 
the standard Berkeley segmentation dataset (BSD) \cite{martin2001}.
This dataset contains 200 various test images of size $321{\times}481$ pixels.
For each image, human segmentations are provided and considered as ground truth
to evaluate the respect of image objects.
At least five decompositions are provided per image, and the presented results
in the following sections are averaged on all ground truths.
Note that other datasets can be considered, \emph{e.g.}, Ref. \citenum{gould2009decomposing,yamaguchi2012parsing},
but the BSD \cite{martin2001} is the most used dataset for comparing superpixel methods.
Moreover, Ref. \citenum{stutz2016} shows that decomposition algorithms that perform well on the BSD
usually perform well on other datasets.

\subsection{Considered Superpixel Methods\label{sec:cons_methods}}

Many frameworks have been proposed to decompose an image into superpixels using 
either graph-based \cite{felzenszwalb2004,veksler2010,liu2011,buyssens2014}, watershed \cite{vincent91,neubert2014compact,machairas2015}, 
coarse-to-fine \cite{vandenbergh2012,yao2015} or gradient-ascent \cite{vedaldi2008,levinshtein2009,achanta2012,li2015,giraud2017_scalp} approaches 
(see Ref. \citenum{stutz2016} for a detailed review of existing methods).
In order to illustrate the proposed evaluation framework, 
and to validate our regularity measure,
we consider the following state-of-the-art superpixel methods:
TP \cite{levinshtein2009}, 	ERS \cite{liu2011}, 		SLIC \cite{achanta2012},  SEEDS \cite{vandenbergh2012}, 	
ERGC \cite{buyssens2014}, 	WP \cite{machairas2015}, 	LSC \cite{li2015}, 		ETPS \cite{yao2015}  and 	SCALP \cite{giraud2017_scalp}.
As stated in the introduction, 
only the most recent methods,
SLIC \cite{achanta2012}, 
ERGC \cite{buyssens2014}, 	WP \cite{machairas2015}, 	LSC \cite{li2015}, 		ETPS \cite{yao2015}  and 	SCALP \cite{giraud2017_scalp},
enable
to set a regularity parameter.
A decomposition example for each considered method is given in Figure \ref{fig:methods_ex}.

\newcommand{\hsp}{0.125\textwidth}
\newcommand{\wsp}{0.19\textwidth}
\begin{figure*}[t!]
\begin{center}
\includegraphics[width=.98\textwidth]{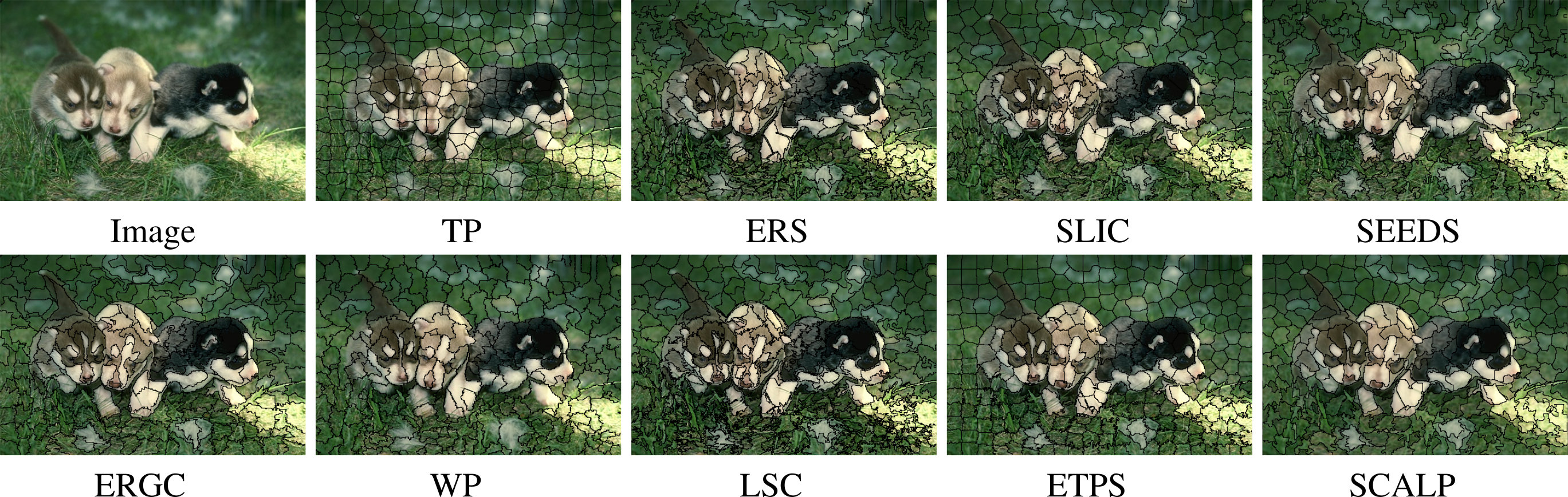}
\end{center}
\caption{
Decomposition example of each considered state-of-the-art superpixel methods for approximately $250$ superpixels.
} 
\label{fig:methods_ex}
\end{figure*}

\subsection{Quantitative Results}

\newcommand{\sstab}{4mm}

In this section, 
we compare the methods considered in Section \ref{sec:cons_methods}
on the recommended metrics for several superpixel scales and
several regularity levels.
In Figure \ref{fig:soa}, we perform the standard evaluation of performances,
according to the number of generated superpixels.
The behavior of each method regarding the different decomposition aspects, 
homogeneity of color clustering (EV), respect of image objects (ASA), and regularity (GR), 
is respectively evaluated in Figures \ref{fig:soa}(a), (b) and (c).
Methods such as TP \cite{levinshtein2009} and WP \cite{machairas2015},
that produce very regular superpixels appear to poorly perform on other metrics.
Although they report high regularity, recent methods SCALP \cite{giraud2017_scalp} and ETPS \cite{yao2015}
perform well on color homogeneity, evaluated with EV.
SCALP \cite{giraud2017_scalp} even performs best on respect of image objects, evaluated with ASA.
We summarize this evaluation
in Table \ref{table:soa}.
A superpixel decomposition cannot obtain maximum values on each criteria at the same time, and
according to the desired decomposition aspect, one would grant more consideration to a particular criteria.

\newcommand{\hspsoac}{0.20\textwidth}
\newcommand{\wspsoac}{0.32\textwidth}

\begin{figure*}[t!]
\begin{center}
 \includegraphics[width=.98\textwidth]{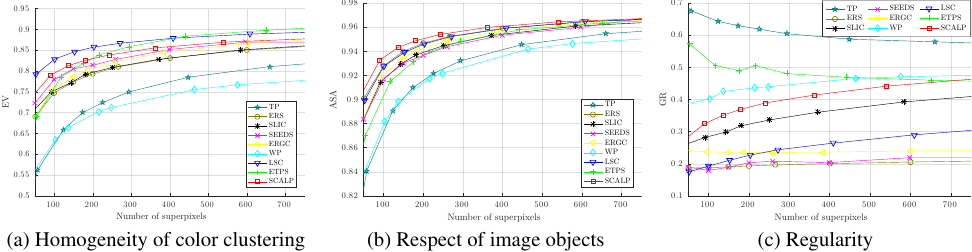}
\end{center}
\caption{
Evaluation of state-of-the-art superpixel methods on EV, ASA and GR according to the number of superpixels, with the 
methods default regularity settings.
The hierarchy between the performances of methods tends to be consistent on all metrics.
} 
\label{fig:soa}
\end{figure*}

\begin{table*}[ht!]
 \caption{
Average EV, ASA and GR on several scales $K=[25,1000]$, with the methods default regularity settings.
} 
 \label{table:soa}
\begin{center}
  {\footnotesize
\renewcommand{\arraystretch}{1.2}
 \begin{tabular}{@{\hspace{1mm}}c@{\hspace{0mm}}l@{\hspace{5mm}}c@{\hspace{\sstab}}c@{\hspace{\sstab}}c@{\hspace{\sstab}}c@{\hspace{\sstab}}c@{\hspace{\sstab}}c@{\hspace{\sstab}}c@{\hspace{\sstab}}c@{\hspace{\sstab}}c@{\hspace{1mm}}}
 \hline
& &TP \cite{levinshtein2009}& ERS \cite{liu2011}& SLIC \cite{achanta2012}&SEEDS \cite{vandenbergh2012}&ERGC \cite{buyssens2014}&WP \cite{machairas2015}&LSC \cite{li2015}&ETPS \cite{yao2015} &SCALP \cite{giraud2017_scalp}\\ \hline
  &EV \eqref{ev}&${\input{ev_TP_0.txt}}$&${\input{ev_ERS_0.txt}}$ &${\input{ev_SLIC_0.txt}}$ &${\input{ev_SEEDS_0.txt}}$ &${\input{ev_ERGC_0.txt}}$ &${\input{ev_WP_0.txt}}$ &${\input{ev_LSC_0.txt}}$ &${\input{ev_ETPS_0.txt}}$ &${\input{ev_SCALP_0.txt}}$ \\
  &ASA \eqref{asa}& ${\input{asa_TP_0.txt}}$&${\input{asa_ERS_0.txt}}$ &${\input{asa_SLIC_0.txt}}$ &${\input{asa_SEEDS_0.txt}}$ &${\input{asa_ERGC_0.txt}}$ &${\input{asa_WP_0.txt}}$ &${\input{asa_LSC_0.txt}}$ &${\input{asa_ETPS_0.txt}}$ &${\input{asa_SCALP_0.txt}}$ \\
     &GR \eqref{gr}&${\input{gr_TP_0.txt}}$&${\input{gr_ERS_0.txt}}$ &${\input{gr_SLIC_0.txt}}$ &${\input{gr_SEEDS_0.txt}}$ &${\input{gr_ERGC_0.txt}}$ &${\input{gr_WP_0.txt}}$ &${\input{gr_LSC_0.txt}}$ &${\input{gr_ETPS_0.txt}}$ &${\input{gr_SCALP_0.txt}}$ \\
  \hline
 \end{tabular}
 }
 \end{center}
\end{table*}

Nevertheless, such evaluation is limited, since 
the hierarchy between methods performances tends to be consistent, \emph{i.e.},
there are no fluctuation in the ranking of methods over a certain number of generated superpixels.
Moreover, as stated in the introduction, 
the setting of regularity has a crucial impact on performances.
Therefore, the evaluation should be performed on this criteria,
to measure the ability of each method to produce the best trade-off between all evaluation aspects.
In Figure \ref{fig:soa_c},
we fix the superpixel scale to $K=250$, and we generate the superpixel results for several regularity levels.
The number of superpixels must be carefully set, since some methods may produce less superpixels that 
the required number.
Larger markers correspond to results obtained with default regularity settings.
Hence, for TP \cite{levinshtein2009}, ERS \cite{liu2011} and SEEDS \cite{vandenbergh2012} methods, 
only one value is reported since these methods do not enable to set the regularity of the decomposition.

\begin{figure}[h!]
\begin{center}
 \includegraphics[width=0.96\textwidth]{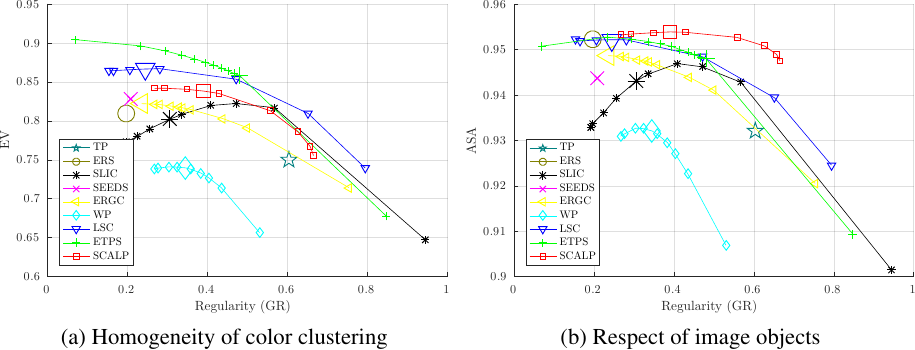}
\end{center}
\caption{
Evaluation of superpixel state-of-the-art superpixel methods on EV, ASA according to the regularity (GR) for $K=250$.
} 
\label{fig:soa_c}
\end{figure}

First, the results appear to be greatly impacted by the regularity setting, 
so for each method, a wide range of performances can be obtained.
Second, contrary to the evaluation according to the number of superpixels, 
the performance curves of methods cross each others, and the hierarchy between methods
can be different according to the regularity parameter.
Third, the default setting of most methods is close from a good trade-off 
between all decomposition aspects but is not necessarily optimal.
For instance, ETPS \cite{yao2015} default setting sets high regularity and 
is not optimal for EV and ASA.
The performances of SLIC \cite{achanta2012} can also be improved on all aspects 
for a higher regularity setting.
We can note that contrary to other methods, results on EV and ASA significantly drop with lower regularity for SLIC.
This can be explained by the post-processing step that 
enforces the pixel connectivity within superpixels \cite{achanta2012}. 
With relaxed regularity constraint, more pixels are disconnected and less accurately grouped during the post-processing.

This evaluation demonstrates that the regularity setting is a crucial parameter to set for all methods.
Hence, a huge bias can appear in the comparison process if this parameter is not optimal, 
or at least set according to the default setting reported by the authors.
This is particularly true since the gap of performances may be very small between methods.
By carefully comparing performances on the recommended metrics, and on
both produced superpixel number and regularity, 
we ensure accurate evaluation of superpixel methods. 
We further illustrate the impact of regularity in the following Section \ref{sec:impact_regu},
where we measure application performances
for several regularity levels.

\section{Impact of the Regularity Parameter on Application Performances\label{sec:impact_regu}}

In this section,
we consider several applications,
and demonstrate how the regularity can affect the performances.
For most applications and superpixel methods,
the performances appear to be highly correlated to the proposed GR metric.
For nearest neighbor matching or tracking, more regular decompositions enable to capture more similar patterns
between images,
while for image compression or contour detection applications,
lower regularity constraint may enable to reach higher performances.

\subsection{Nearest Neighbor Matching\label{sec:nearest}}

In this section, we investigate the impact of regularity on the superpixel matching accuracy between images.
The correlation between the superpixel shape and parsing performances,
\emph{i.e.}, segmentation and labeling, has been demonstrated in Ref. \citenum{strassburg2015influence},
where regular grid-based  methods appear to give higher pixel-wise results.
We propose to compute superpixel matching between regular and irregular decompositions
using the superpixel neighborhood structure introduced in Ref. \citenum{giraud2017spm}.
In Figure \ref{fig:nn_expe},
two decompositions are computed on the same
image using both regular
(Figures \ref{fig:nn_expe}(a) and (b))
and irregular
decompositions (Figures \ref{fig:nn_expe}(d) and (e)). 
Correspondences are then computed between the two regular decompositions 
and the two irregular ones, and we compare the respective displacement magnitude
of the matches in Figures \ref{fig:nn_expe}(c) and (f).
The correspondences are computed on color features (normalized color histograms), 
and we display the magnitude of the matches displacement with the standard optical 
flow representation (Figure \ref{fig:nn_expe}(g)).

\newcommand{\wtabss}{0.125\textwidth}
\newcommand{\wtabsss}{0.20\textwidth}
\begin{figure*}[t]
\begin{center}
 \includegraphics[width=0.98\textwidth]{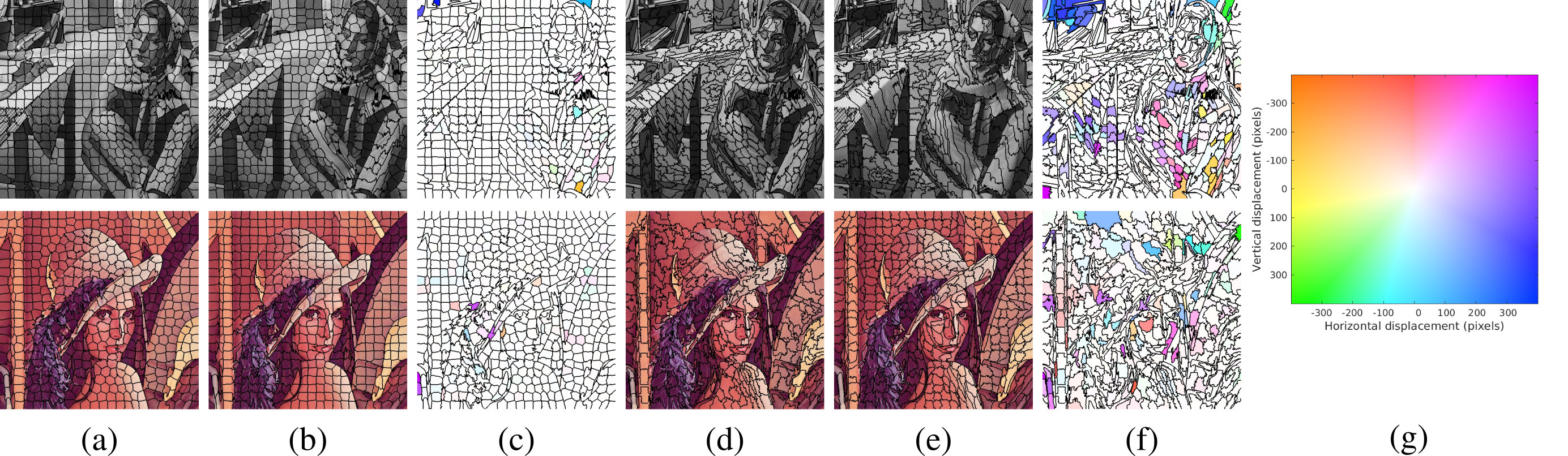}
\end{center}
\caption{
Examples of matching between regular (a), (b) and irregular (d), (e) decompositions. 
The magnitude of the displacement between the barycenters of the matched superpixels is respectively given
for regular and irregular matching in (c) and (f). 
The displacement is illustrated with the standard optical flow representation (g),
and respectively displayed within the superpixel boundaries of images (a) and (d).
On these examples, the average GR is respectively $0.478$ and $0.205$  
for the regular (a) and (b) and irregular (d) and (e) decompositions.
} 
\label{fig:nn_expe}
\end{figure*}

While the correspondences between regular decompositions are accurate (low displacement between matched superpixels),
irregular decompositions mislead the matching.
With irregular methods, the barycenters may fall outside the superpixels, 
making difficult the use of the geometrical information of the neighborhood.
Regular methods
produce decompositions where all
elements globally have the same size and 
the barycenters are usually contained into the superpixels.
Therefore, the spatial relation between a superpixel barycenter 
and the ones of its neighbors is more relevant,
which enables more accurate matching.

\subsection{Superpixel Tracking}

The computational gain obtained with superpixels is particularly interesting 
in real-time application such as video tracking \cite{wang2011,chang2013,reso2013}.
Two main approaches exist for superpixel-based video segmentation and tracking.
The first one computes 3D superpixels, \emph{i.e.}, {supervoxels}, with the same properties
than for 2D methods. 
Hence, it requires the whole sequence to produce a decomposition and cannot be directly computed 
with the acquisition of new frames.

The second approach computes superpixels on each frame
and the decomposition of the following ones are inferred from the previous segmentations \cite{chang2013}. 
In these methods, optical flow computation is generally used to 
transfer the superpixels barycenter from a frame to another one \cite{chang2013}.
As a consequence, a superpixel can only have a unique correspondence in the next frame,
and elements describing an object can then be tracked across the video sequence.

We investigate the impact of regularity on the tracking performances of such approach 
with the temporally consistent superpixels (TSP) method \cite{chang2013}.
We consider images from Ref. \citenum{tsai2012}, which is composed of short
video sequences provided with binary labeling ground truth, 
and we compare the tracking performances obtained with TSP \cite{chang2013}
generating regular and irregular decompositions.
In Table \ref{table:tracking}, we report the pixel-wise labeling accuracy over the whole sequences,
which expresses the ability of the method to track superpixels.
With more regular decompositions, 
the superpixel shapes and sizes are consistent over time, which
enhances the tracking accuracy.
We also report the average percentage of new superpixels created between two decompositions.
Hence, higher values for irregular decompositions indicate the loss of tracking during the process.
Examples of tracking for such regular and irregular decompositions are illustrated in Figure \ref{fig:tracking_ex} 
for the considered sequences.
The consistency of the regular decompositions enables to more efficiently track the superpixels,
\emph{i.e.}, to compute one-to-one correspondences
through the sequences.

\begin{table}[h!]

\caption{
Tracking accuracy with TSP \cite{chang2013}
on sequences from Ref. \citenum{tsai2012}.
The labeling accuracy and the percentage of lost superpixels over frames are reported for
regular and irregular settings.
}
\renewcommand{\arraystretch}{1.3}
\begin{center}
{\small
\begin{tabular}{@{\hspace{0mm}}l@{\hspace{4mm}}@{\hspace{4mm}}c@{\hspace{4mm}}c@{\hspace{8mm}}c@{\hspace{4mm}}c@{\hspace{4mm}}}
\cline{2-5}
&\multicolumn{2}{@{\hspace{0mm}}c@{\hspace{0mm}}}{{\hspace{-4mm}}Covering accuracy{\hspace{1mm}}}&\multicolumn{2}{c}{{\hspace{-3mm}}Lost tracking}  \\ \cline{2-5}
& {\hspace{4mm}} Regular & Irregular& Regular & Irregular\\ \hline
\multicolumn{1}{@{\hspace{2mm}}l@{\hspace{3mm}}}{birdfall2} 	&$98.3\%$	&$97.8\%$	&$1.0\%$	&$1.4\%$ \\ 
\multicolumn{1}{@{\hspace{2mm}}l@{\hspace{3mm}}}{girl}      	&$51.1\%$	&$50.4\%$	&$13.9\%$	&$24.8\%$\\
\multicolumn{1}{@{\hspace{2mm}}l@{\hspace{3mm}}}{parachute}	&$75.3\%$	&$73.9\%$	&$4.5\%$	&$5.1\%$\\
\multicolumn{1}{@{\hspace{2mm}}l@{\hspace{3mm}}}{penguin}	&$94.3\%$	&$85.0\%$	&$2.6\%$	&$8.8\%$\\ \hline
\multicolumn{1}{@{\hspace{2mm}}l@{\hspace{3mm}}}{{average}}	&$79.8\%$	&$76.7\%$	&$5.5\%$	&$10.0\%$\\ \hline
\end{tabular} 
}
\end{center}
\label{table:tracking}
\end{table}

\newcommand{\htabmh}{0.095\textwidth}
    \newcommand{\wtabb}{0.15\textwidth}
\begin{figure*}[ht!]
\begin{center}
 \includegraphics[width=0.96\textwidth]{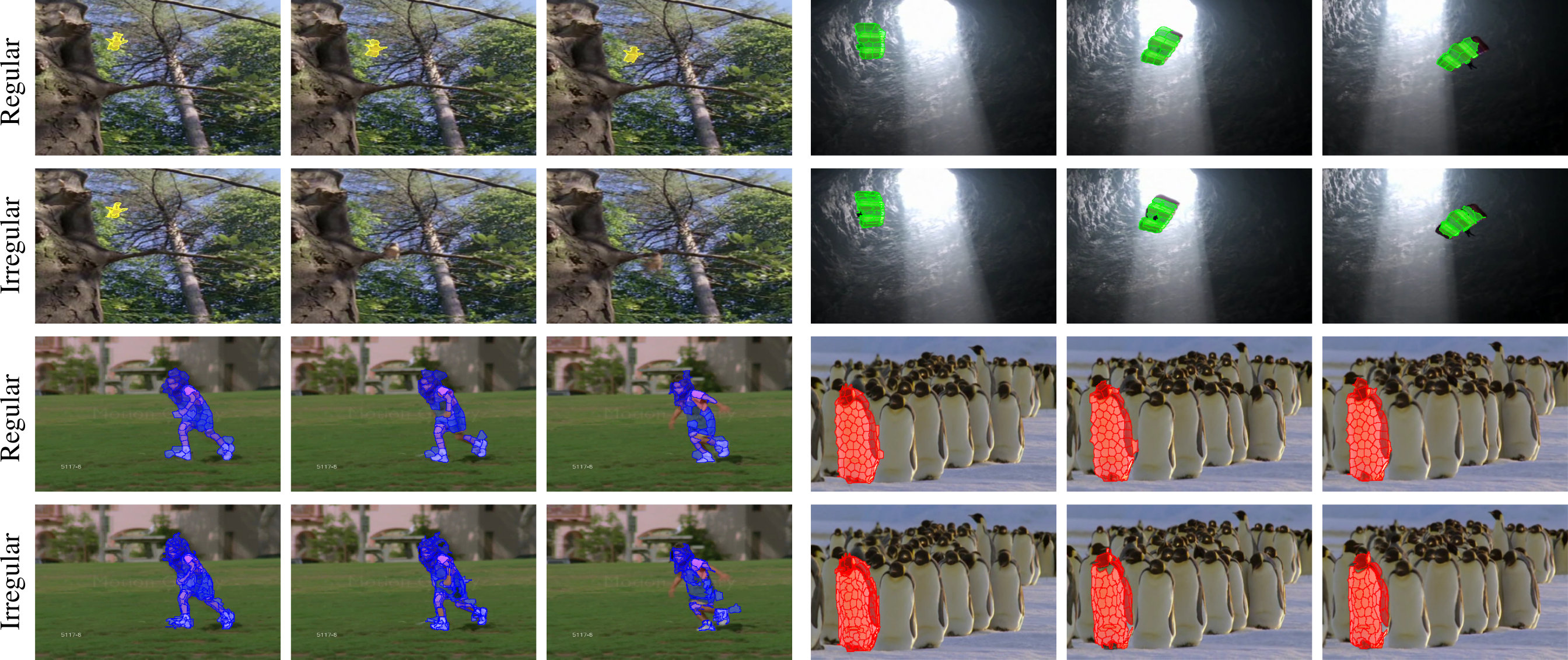}
\end{center}
\caption{
Examples on superpixel tracking performances of TSP \cite{chang2013} over the sequences of Ref. \citenum{tsai2012}
for regular and irregular decompositions.
On these examples, the average GR is respectively $0.308$ and $0.554$  for the irregular and regular decompositions.
} 
\label{fig:tracking_ex}
\end{figure*}

\subsection{Image Compression}

Since superpixels define a lower-level representation, 
they can be used to compress the image content \cite{levinshtein2009,wang2013}.
We propose to study the impact of regularity on the image compression efficiency.
Similarly to Ref. \citenum{wang2013},  each
superpixel color channel is approximated
by a polynomial of third
order. 
The color information is then contained into the polynomial coefficients
which are computed by a least mean square approach.

A compression example is illustrated in Figure \ref{fig:ic_ex}
for a regular and an irregular decomposition.
We show the reconstructed image $I_r$ from the polynomial coefficients stored for each superpixel.
The irregular decomposition produces a reconstructed image closer to the 
initial one $I$, while the boundaries of superpixels are visually more noticeable with the regular one.
To further evaluate the compression efficiency, we report in Figure \ref{fig:ic}
the mean square error (MSE) between the initial image
and the reconstructed one:
\begin{equation}
 \text{MSE}(I,I_r) = \frac{1}{|I|}\sum_{p\in I}\left(I(p)-I_r(p)\right)^2. \label{mse}
\end{equation}
The decompositions of state-of-the-art superpixel methods 
are computed at several regularity levels, for $K=250$ superpixels, 
and the MSE \eqref{mse} results are averaged on all BSD images.
For most methods, the default regularity parameter is not optimal to perform well on image compression.
These results are in line with the evaluation on the EV metric  in Figure \ref{fig:soa_c}(a), 
since homogeneous superpixels in terms of colors
enable to compute  polynomial coefficients
that more accurately fit to the initial image content.
However, such decompositions are obtained at the expense of lower regularity.

  \begin{figure}[t!]
\begin{center}
 \includegraphics[width=0.90\textwidth]{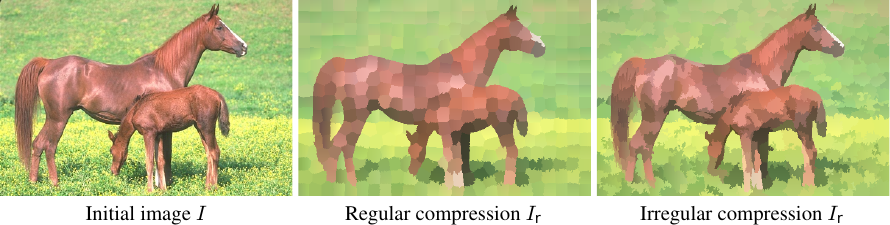}
\end{center}
  \caption{
  Example of image compression with regular and irregular decompositions for $K=250$ superpixels.
  The reconstructed colors are displayed within each superpixel.
  } 
  \label{fig:ic_ex}
  \end{figure}

    \begin{figure}[t!]
\begin{center}
{\footnotesize
 \begin{tabular}{@{\hspace{0mm}}c@{\hspace{0mm}}}
 \includegraphics[width=0.68\textwidth,height=0.42\textwidth]{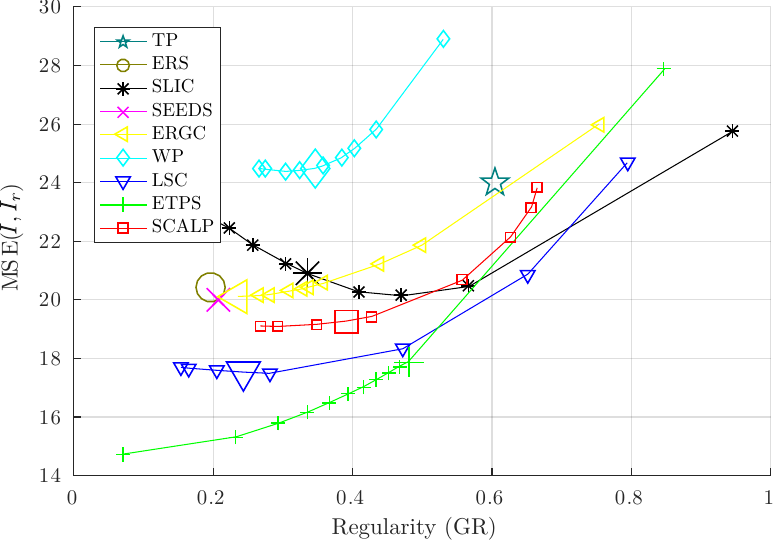}
  \end{tabular}
  }
  \end{center}
  \caption{
  Results of state-of-the-art superpixel methods on image compression.
The average MSE \eqref{mse} between initial and reconstructed BSD images are reported.
  } 
  \label{fig:ic}
  \end{figure}

\subsection{Contour Detection and Segmentation}

Finally, contour detection and segmentation performances 
are studied for several levels of shape constraint.
The ability of superpixel methods to provide accurate contour detection
was investigated, for instance in Ref. \citenum{arbelaez2011,vandenbergh2012,giraud2017_scalp}.
A superpixel decomposition can naturally produce more regions than the number of image objects,
and the number of superpixel boundaries is dependent on the superpixel scale.
Nevertheless, accurate contour map can be easily obtained
by averaging the boundaries of superpixel decompositions obtained at multiple scales \cite{vandenbergh2012,giraud2017_scalp}.
Hence, the average map has high values at contour pixels that are detected at several scales
and is robust to texture areas and high local gradients.

As stated in Section \ref{sec:rio},
the boundary recall (BR) measure is not sufficient to express the contour detection performance, and should be 
compared to the precision (P) measure, which evaluates the percentage of accurate detection
among the computed superpixel contour.
Such contour detection accuracy can be synthesized by the F-measure defined as:
 \begin{equation}
 \text{F}=\frac{2.\text{P.BR}}{\text{P}+\text{BR}} . \label{fmeasure}
 \end{equation}
In the following, superpixel boundaries are computed at several scales, from $K=[25,1000]$, 
and averaged to compute the contour map.
In Figure  \ref{fig:seg_res}(d),
we report
the precision-boundary recall curves \cite{martin2004} corresponding to the 
contour detection obtained with the default regularity parameter of each method.
SCALP \cite{giraud2017_scalp} and SLIC \cite{achanta2012} appear to obtain the best contour detection performances.

  \newcommand{\hspsoa}{0.36\textwidth}
\newcommand{\wspsoa}{0.46\textwidth}
  
\begin{figure*}[t!]
\begin{center}
 \includegraphics[width=0.96\textwidth]{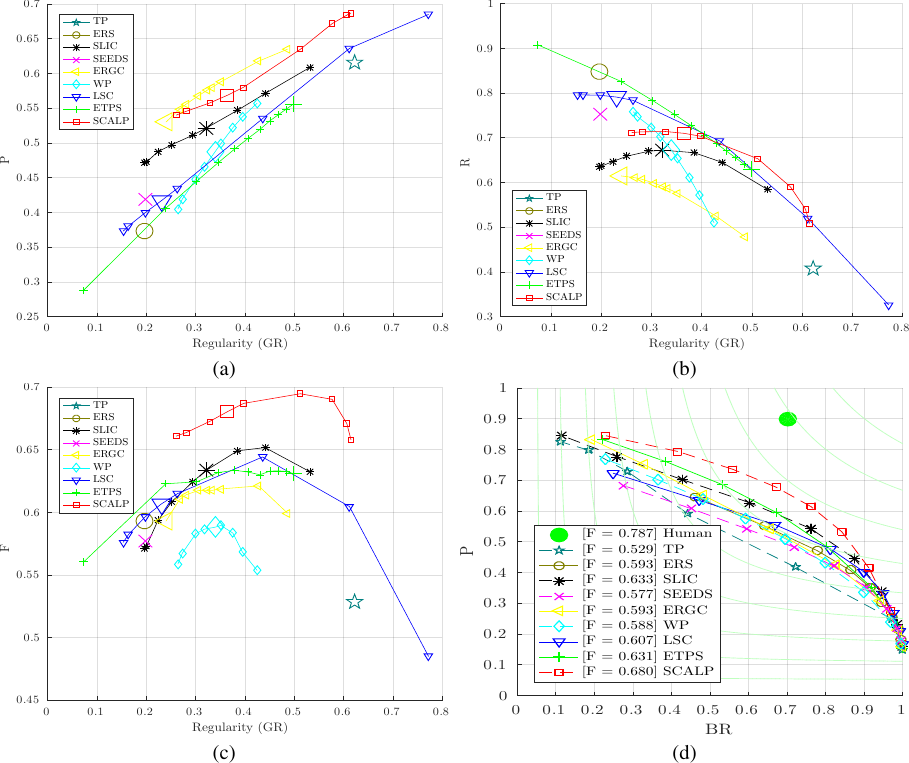}
\end{center}
\caption{
Contour detection performances according to the regularity with the precision (a), 
boundary recall (b), and maximum F-measure (c).
Precision-recall curves for default regularity setting are illustrated in (d).
} 
\label{fig:seg_res}
\end{figure*}

\begin{figure*}[ht!]
\begin{center}
 \includegraphics[width=0.96\textwidth]{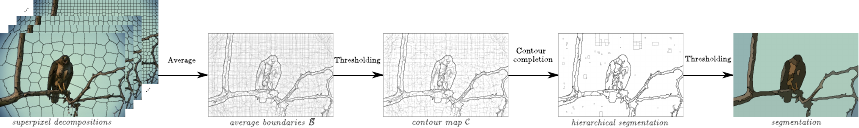}
\end{center}
\caption{
Segmentation from superpixel contour detection pipeline.
Superpixel decompositions are computed at several scales and their boundaries are averaged to produce a contour map.
The low confidence boundary pixels are removed with a thresholding.
From this contour map, a hierarchical segmentation can be computed and thresholded to provide an image segmentation,
that we illustrate with average colors.
} 
\label{fig:seg_pipeline}
\end{figure*}

To underline the impact of the regularity parameter on contour detection performances, 
the precision (P), boundary recall (BR), and corresponding maximum F-measure are 
computed for several regularity levels respectively in Figures \ref{fig:seg_res}(a), (b) and (c).
The regularity clearly impacts the performances, since the precision and boundary recall 
respectively increases and drops with the GR measure in Figures \ref{fig:seg_res}(a) and (b).
Consequently, a good trade-off for maximum F-measure is obtained for GR values between $0.35$ and $0.45$.
More importantly, results appear to be significantly impacted by the regularity, so most methods could perform better than with their default parameters, on this application, with appropriate settings.

Finally, a segmentation can be generated using
a contour completion step that compute closed regions from the contour map \cite{arbelaez2009}.
The whole contour detection and segmentation pipeline is illustrated in Figure \ref{fig:seg_pipeline}, and 
we show in Figure \ref{fig:seg_ex} contour detection and segmentation examples 
computed from regular and irregular decompositions with Ref. \citenum{achanta2012}.
The average of regular superpixel boundaries does not enable
to efficiently capture the ground truth contours (see Figure \ref{fig:seg_ex}(c)).
The contour closure of the image objects is not robust enough to
lead to a segmented region with Ref. \citenum{arbelaez2009}.
With irregular decompositions, the superpixel shapes are less constrained 
and can better adhere to the image contours.
The contour closure obtained by the average of boundaries is hence more robust and can
lead to accurate segmentation (see Figure \ref{fig:seg_ex}(b)).
Nevertheless, note that the thresholding enables to remove the boundary artifacts generated with irregular decompositions,
that can obtain lower performances in terms of respect of image objects when evaluated with ASA.

        \newcommand{\lastt}{0.195\textwidth}
\begin{figure}[h!]
\begin{center}
 \includegraphics[width=0.94\textwidth,height=1.18\textwidth]{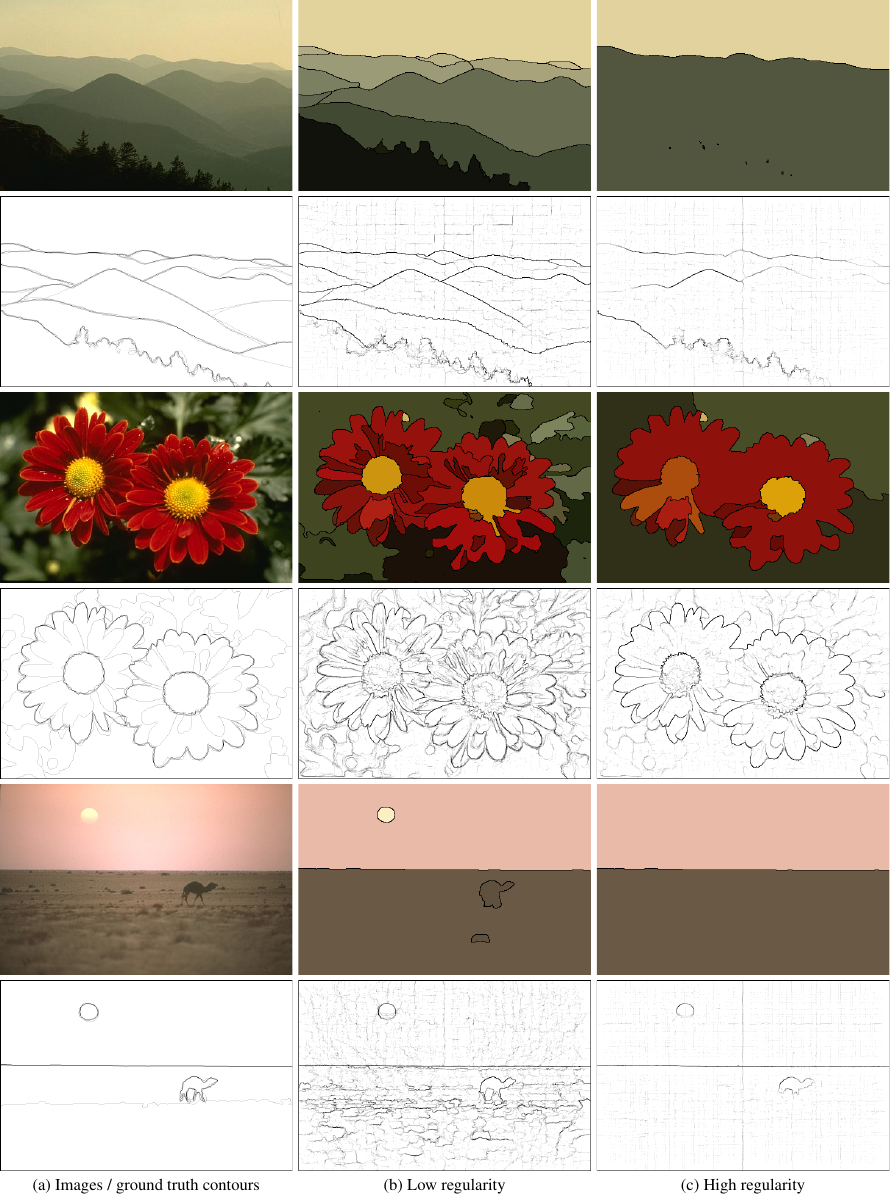}
\end{center}
\caption{
Examples of contour detection and segmentation
with low (b) and high (c) regularity settings. 
Initial images with the average ground truth contours are shown in (a). 
Regions with average colors are illustrated after a thresholding of the hierarchical segmentation
computed from the contour map with Ref. \citenum{arbelaez2009}.
The average GR is respectively $0.256$ and $0.481$  for the irregular and regular decompositions.
} 
\label{fig:seg_ex}
\end{figure}

\subsection{Correlation between GR and Performances\label{sec:pearson}}

In this section, 
we quantitatively show the correlation between the proposed global regularity measure GR 
and the performances of superpixel metrics and applications.
We report in Table \ref{table:pearson},
the Pearson correlation coefficient between all regularity measures: 
C \eqref{circu} \cite{schick2012}, SRC \eqref{src}, J \eqref{jaccard} \cite{machairas2015}, 
SMF \eqref{smf} and GR \eqref{gr},
and the performances of 
ASA \eqref{asa}, UE \eqref{ue_np}, BR \eqref{br}, P \eqref{fmeasure}, EV \eqref{ev} and MSE \eqref{mse}.
The correlations are averaged on the decompositions of 
SLIC \cite{achanta2012}, ERGC \cite{buyssens2014}, WP \cite{machairas2015}, LSC \cite{li2015}, ETPS \cite{yao2015} and SCALP \cite{giraud2017_scalp}
methods, which enable to set a regularity parameter,
and are computed for $K=250$ superpixels.

Logically the ASA, BR, P and EV measures are negatively correlated
while the UE and MSE measures are in line  with the regularity.
The proposed SRC metric globally improves the correlations
compared to the circularity metric C,
and the proposed SMF improves the ones of the J measure.
Overall, the proposed GR measure obtains the highest correlation with
the performances, demonstrating its relevance compared to previously introduced metrics.
Finally, note that the correlation between ASA and UE is in average $-0.9969$,
which further assess the assumption made in Section \ref{sec:rio}.

\begin{table}[t!]
\caption{
Comparison of Pearson correlation coefficient between performances of metrics and applications, and the
different regularity measures.
The average results are computed on the absolute value of the correlations.
}
\begin{center}
\renewcommand{\arraystretch}{1.3}
{\footnotesize
\begin{tabular}{@{\hspace{0mm}}lr@{\hspace{2mm}}r@{\hspace{2mm}}r@{\hspace{2mm}}r@{\hspace{2mm}}@{\hspace{2mm}}r@{\hspace{2mm}}@{\hspace{0mm}}}
\cline{2-6}
&\multicolumn{1}{c}{GR \eqref{gr}}&  \multicolumn{1}{c}{SMF \eqref{smf}}   & \multicolumn{1}{c}{J \eqref{jaccard}} &\multicolumn{1}{c}{SRC \eqref{src}}  & \multicolumn{1}{c}{C \eqref{circu}} \\ \hline
\multicolumn{1}{@{\hspace{2mm}}l@{\hspace{3mm}}}{ASA \eqref{asa}} 	
&$\mathbf{\input{gr_asa.txt}}$	&$\input{smf_asa.txt}$	&$\input{mf_asa.txt}$	&$\input{src_asa.txt}$ 	&$\input{c_asa.txt}$  \\ 
\multicolumn{1}{@{\hspace{2mm}}l@{\hspace{3mm}}}{UE \eqref{ue_np}}      
&$\mathbf{\input{gr_ue.txt}}$	&$\input{smf_ue.txt}$	&$\input{mf_ue.txt}$	&$\input{src_ue.txt}$ 	&$\input{c_ue.txt}$  \\ 
\multicolumn{1}{@{\hspace{2mm}}l@{\hspace{3mm}}}{BR \eqref{br}}		
&$\mathbf{\input{gr_br.txt}}$	&$\input{smf_br.txt}$	&$\input{mf_br.txt}$	&$\input{src_br.txt}$ 	&$\input{c_br.txt}$  \\ 
\multicolumn{1}{@{\hspace{2mm}}l@{\hspace{3mm}}}{P \eqref{fmeasure}}	
&${\input{gr_p.txt}}$	&$\input{smf_p.txt}$	&$\input{mf_p.txt}$	&$\input{src_p.txt}$ 	&$\mathbf{\input{c_p.txt}}$  \\ 
\multicolumn{1}{@{\hspace{2mm}}l@{\hspace{3mm}}}{EV \eqref{ev}}		
&$\mathbf{\input{gr_ev.txt}}$	&$\input{smf_ev.txt}$	&$\input{mf_ev.txt}$	&$\input{src_ev.txt}$ 	&$\input{c_ev.txt}$  \\ 
\multicolumn{1}{@{\hspace{2mm}}l@{\hspace{3mm}}}{MSE \eqref{mse}}	
&$\mathbf{\input{gr_mse.txt}}$	&$\input{smf_mse.txt}$	&$\input{mf_mse.txt}$	&$\input{src_mse.txt}$ 	&$\input{c_mse.txt}$  \\  \hline
\multicolumn{1}{@{\hspace{2mm}}l@{\hspace{3mm}}}{{Average}}		
&$\mathbf{\input{gr_avg.txt}}$	&$\input{smf_avg.txt}$	&$\input{mf_avg.txt}$	&$\input{src_avg.txt}$ 	&$\input{c_avg.txt}$  \\  \hline
\end{tabular} 
}
\end{center}
\label{table:pearson}
\end{table}

\section{Discussion}

Since the superpixel representation has become popular with works such as Ref. \citenum{ren2003},
many works have proposed decomposition methods using various techniques.
In the addition to the high number of superpixel works, 
specific metrics dedicated to superpixel evaluation
were progressively introduced and non exhaustively reported in
the associated works.
More generally,
to evaluate the improvement over state-of-the-art of new proposed methods, 
it is crucial to have a consistent and rigorous evaluation of the reported results.
Otherwise, we cannot relevantly assess the potential breakthrough performances of a new method.
These issues of
evaluation metrics 
finds an echo in other image processing applications.
For instance, in most of the denoising literature, the SSIM is reported along with
the PSNR metric to relevantly compare the performances, since they are complementary.

In the superpixel context, 
most recent methods only require to set the number of generated superpixels and
to tune a regularity parameter.
By reporting results evaluated according to both aspects, on the advocated metrics, 
we ensure that the comparison is not biased by the set parameter,
and it more globally expresses the range of possible decompositions that can be obtained.
Moreover, the main aim of superpixels is to be used as a  pre-processing to speed 
up the computational time of a given application.
The number of superpixels corresponds to the number of elements to process, and is generally in line
with the complexity of the application.
Therefore, the number of desired superpixels is set by the time constraint, 
and the main parameter to study is the regularity.

The evaluation of performances according to the GR measure
offers a new and relevant interpretation of the potential of superpixel methods,
since it is directly correlated
with the results of image processing applications.
We quantitatively demonstrate in Section \ref{sec:pearson}, that our measure 
is more correlated to the application performances than previously introduced regularity metrics.
In Section \ref{sec:impact_regu}, we show that
regular decompositions may  better perform for applications such as nearest neighbor matching or tracking. 
The adjacency relations between regular superpixels are more consistent between images,
which enables
to robustly describe superpixels and their neighborhood to match similar patterns \cite{giraud2017spm}.
When processing the image itself, with applications such as image compression or contour detection,
more irregular decompositions may be desirable.
By relaxing the regularity constraint, the superpixels group more homogeneous pixels in terms of color,
which lead to more accurate compression,
and more efficiently follows the image contours, 
leading to higher contour detection.
Consequently, according to the tackled application, we demonstrate that 
the regularity must be controlled to provide the best possible results.

As shown throughout the paper, and illustrated in Figure \ref{fig:slic_ex},
the same method can produce decompositions of various regularity,
and perform well on several applications.
Therefore, although many papers make a clear distinction 
between regular and irregular methods, generally according to the
results produced with their default parameter, 
it appears that this property does not reflect the potential behavior of a method.
The distinction between regular and irregular methods should
be based on the method ability to produce superpixels of different sizes, 
instead of irregular shapes.
For instance, 
methods such as Ref. \citenum{felzenszwalb2004,rubio2016} can produce regions of very different sizes, while
 SLIC \cite{achanta2012}, and most methods that are based on,
can produce irregular shaped superpixels by relaxing the regularity constraint, 
but the clustering is still constrained in a square bounding box,
so the produced superpixels have consistent sizes.

\section{Conclusion}

In this work, we propose an evaluation framework of
superpixel decomposition methods.
We take a global view of existing superpixel metrics, and
investigate their limitations to propose a relevant comparison process.
We reduce the evaluation to
the three main aspects of a decomposition: 
the homogeneity of color clustering with EV, the respect of the image objects with ASA, 
and the regularity with the proposed GR.
Contrary to existing regularity metrics, the proposed GR
measure relevantly evaluates both shape regularity and consistency of the superpixels, 
addressing the non-robustness of state-of-the-art metrics.

Since most of the recent methods enable to set a shape parameter,
which impacts the decomposition performances,
we also compare methods according to their regularity, evaluated with the GR measure.
Hence, the reported performances are not biased by the methods settings.
The GR metric is highly correlated to the performances of 
applications such as matching, tracking or contour detection,
where the same method can perform well on several applications, 
according to the chosen regularity parameter.
Therefore, we advocate to use the proposed framework and metrics
when comparing superpixel or supervoxel methods, and
to evaluate the performances at several regularity levels evaluated with GR.

\appendix

\section{Correlation between ASA and UE\label{sec:asa_demo}}

In this section, we demonstrate the relation \eqref{asa_ue}
for several assumptions.

\begin{proposition} 
If all superpixels $S_k$ have a major overlap, \textit{i.e.}, 
if $\forall S_k, \exists  G_i, |S_k\cap G_i| \geq |S_k|/2, \text{ then,}$
$\text{ASA}(\SSS,\GG) = 1 - \text{UE}(\SSS,\GG)/2$.
\end{proposition}

\begin{proof}
\noindent The major overlap hypothesis, illustrated in Figure \ref{fig:cond_ue_asa}(b),
is that each superpixel $S_k$ has a major overlap with a region $G_i$
that covers more than half of its area, \emph{i.e.}, 
\begin{equation}
 \forall S_k, \exists  G_i, |S_k\cap G_i| \geq |S_k|/2. \label{cond_ue_asa}
\end{equation} 
\noindent Hence, we have $\min\{|S_k\cap G_i|,|S_k\backslash G_i|\} = |S_k\backslash G_i|$, and
$\forall G_{j,j\neq k}$,  
$\min\{|S_k\cap G_j|,|S_k\backslash G_j|\} = |S_k\cap G_j|$. 

\noindent Therefore, we obtain,
\begin{align*}
\sum_{G_j}\min\{|S_k\cap G_j|,|S_k\backslash G_j|\} 
    &= |S_k\backslash G_i| + \sum_{G_j, j\neq k}|S_k\cap G_j| ,  \\     
    &= (|S_k| - |S_k\cap G_i|) + (|S_k| - |S_k\cap G_i|) ,\\
    & = 2(|S_k| - |S_k\cap G_i|),\\
    &= 2(|S_k| - \max_{G_j}|S_k\cap G_j|).
\end{align*}
\noindent Consequently, we have,
\begin{align*}
\text{UE}(\SSS,\GG) &=  \frac{1}{|I|}\sum_{S_k}\sum_{G_j}\min\{|S_k\cap G_j|,|S_k\backslash G_j|\} , \\   
    &= \frac{2}{|I|}\sum_{S_k}\left(|S_k| - \max_{G_j}|S_k\cap G_j|\right) ,\\
    &= 2\left(1 - \frac{1}{|I|}\sum_{S_k}\max_{G_j}|S_k\cap G_j|\right) ,\\
    &= 2(1-\text{ASA}(\SSS,\GG)) .
\end{align*}
\noindent Finally, according to the ASA definition \eqref{asa}, the relation \eqref{asa_ue} is verified.
\end{proof}

 \begin{proposition} 
 If all superpixels $S_k$ overlap with at most $2$ regions, \textit{i.e.}, 
 if $\forall S_k, |I_k|\leq$ $2$, with $\sum_{i\in I_k}|S_k\cap G_i|$ $=$ $|S_k|, 
 \text{ then, } \text{ASA}(\SSS,\GG) = 1 - \text{UE}(\SSS,\GG)/2$.
\end{proposition}

\begin{proof}
In this case, a superpixel $S_k$ can overlap with one or two regions. 
\noindent If $S_k$ overlaps with only one region, \emph{e.g.}, $G_i$, then the assumption in
 \eqref{cond_ue_asa} is verified.
Otherwise, if $S_k$ overlaps with two regions, \eqref{cond_ue_asa} is also
necessarily true, since one of the two regions 
overlaps with at least half of $S_k$.
Hence, in both cases,
\eqref{asa_ue} is verified.
\end{proof}

\begin{proposition} 
If $\GG$ is a binary ground truth, \emph{i.e.}, $|\GG|$ $\leq$ $2$, 
then, $\text{ASA}(\SSS,\GG) = 1 - \text{UE}(\SSS,\GG)/2$.
\end{proposition}

\begin{proof}
 If $\GG$ is a binary ground truth, it is composed of only two regions
such that $\GG = G_1 \cup G_2$, and all superpixels necessarily
overlap with at most $2$ regions.
The assumption of Proposition 2 is hence true and \eqref{asa_ue} is verified.
\end{proof}

 Figure \ref{fig:ue_asa} demonstrates the relevance of relation \eqref{asa_ue}.
 Very low differences are reported between this model and measured ASA and UE for 
 decompositions of the following state-of-the-art methods: 
 TP \cite{levinshtein2009}, 	ERS \cite{liu2011}, 		SLIC \cite{achanta2012},   SEEDS \cite{vandenbergh2012}, 
ERGC \cite{buyssens2014}, 	WP \cite{machairas2015}, 	LSC \cite{li2015}, 		ETPS \cite{yao2015}  and 	SCALP \cite{giraud2017_scalp},
 applied to the test images of the 
 Berkeley segmentation dataset (BSD) \cite{martin2001}.
 Note that the error amplitude logically reduces with the number of superpixels, \emph{i.e.}, 
 the accuracy of the decomposition.

\begin{figure}[ht!]
\begin{center}
\includegraphics[height=0.38\textwidth,width=0.60\textwidth]{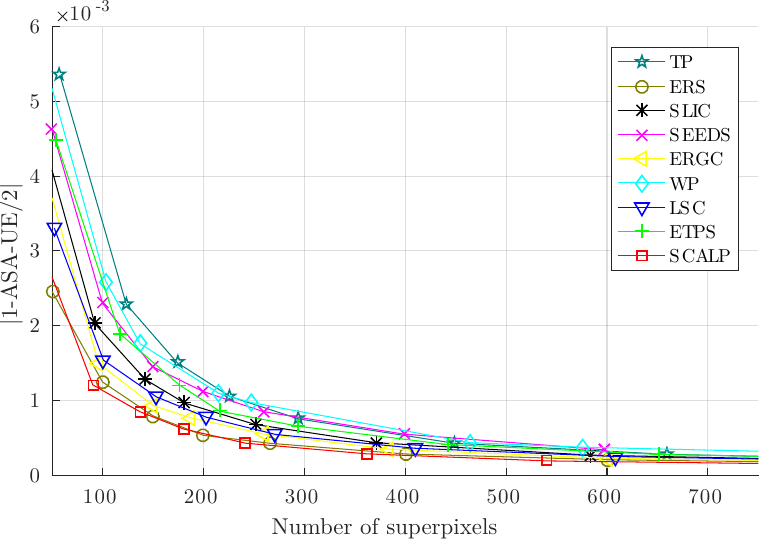}
\end{center}
\caption{
Error between measured ASA and UE and the relation  \eqref{asa_ue}.
The low values compared to ASA and UE metrics demonstrate the relevance of our hypothesis.
} 
\label{fig:ue_asa}
\end{figure}

\section*{Acknowledgments}
This study has been carried out with financial support from the French 
State, managed by the French National Research Agency (ANR) in the  
frame of the GOTMI project (ANR-16-CE33-0010-01) and the
Investments for the future Program IdEx Bordeaux 
(ANR-10-IDEX-03-02) with the Cluster of excellence CPU.

\bibliographystyle{spiejour}
\bibliography{JEI}

\vspace{0.3cm}\noindent \textbf{Rémi Giraud} received the M.Sc. in telecommunications at ENSEIRB-MATMECA School of Engineers, 
and the M.Sc. in signal and image processing from the University of Bordeaux, France, in 2014.
Since, he is pursuing his Ph.D. at Laboratoire  Bordelais  de  Recherche   en
Informatique in the field of image processing.
His research areas mainly include computer vision and image processing applications
with non-local methods and superpixel representation,
applied to natural and medical images.
\\

\noindent \textbf{Vinh-Thong Ta} received the M.Sc. and Doctoral degrees in computer science from 
the University of Caen Basse-Normandie, France, in 2004
and 2009, respectively. From 2009 to 2010, he was an Assistant Professor in computer science
with the School of Engineers of Caen, France.
Since 2010, he is an Associate Professor with the Computer Science Department, School of Engineers
ENSEIRB-MATMECA. His research mainly concerns image and
data processing. \\

\noindent \textbf{Nicolas Papadakis} received the Graduate degree in
applied mathematics from the National Institute of Applied Sciences, Rouen, France, in 2004, and the 
Ph.D. degree in applied mathematics from the University of Rennes, France, in 2007. 
He is currently a  Researcher from the Centre National de la Recherche Scientifique, Institut de Mathématiques
de Bordeaux, France. His main research interests include tracking, motion estimation, and optimal
transportation for image processing and data assimilation problems.

\listoffigures
\listoftables

\end{spacing}
\end{document}

%% file: ev_TP_0.txt
0.734

%% file: ev_ERS_0.txt
0.807

%% file: ev_SLIC_0.txt
0.809

%% file: ev_SEEDS_0.txt
0.829

%% file: ev_ERGC_0.txt
0.820

%% file: ev_WP_0.txt
0.713

%% file: ev_LSC_0.txt
0.863

%% file: ev_ETPS_0.txt
0.845

%% file: ev_SCALP_0.txt
0.840

%% file: asa_TP_0.txt
0.925

%% file: asa_ERS_0.txt
0.951

%% file: asa_SLIC_0.txt
0.944

%% file: asa_SEEDS_0.txt
0.943

%% file: asa_ERGC_0.txt
0.947

%% file: asa_WP_0.txt
0.922

%% file: asa_LSC_0.txt
0.950

%% file: asa_ETPS_0.txt
0.947

%% file: asa_SCALP_0.txt
0.956

%% file: gr_TP_0.txt
0.610

%% file: gr_ERS_0.txt
0.197

%% file: gr_SLIC_0.txt
0.343

%% file: gr_SEEDS_0.txt
0.203

%% file: gr_ERGC_0.txt
0.237

%% file: gr_WP_0.txt
0.444

%% file: gr_LSC_0.txt
0.244

%% file: gr_ETPS_0.txt
0.487

%% file: gr_SCALP_0.txt
0.395

%% file: gr_asa.txt
-0.5473

%% file: smf_asa.txt
-0.5250

%% file: mf_asa.txt
-0.5266

%% file: src_asa.txt
-0.5350

%% file: c_asa.txt
-0.5318

%% file: gr_ue.txt
0.5506

%% file: smf_ue.txt
0.5284

%% file: mf_ue.txt
0.5299

%% file: src_ue.txt
0.5384

%% file: c_ue.txt
0.5353

%% file: gr_br.txt
-0.9136

%% file: smf_br.txt
-0.8974

%% file: mf_br.txt
-0.8972

%% file: src_br.txt
-0.9049

%% file: c_br.txt
-0.9034

%% file: gr_p.txt
-0.9627

%% file: smf_p.txt
-0.9645

%% file: mf_p.txt
-0.9656

%% file: src_p.txt
-0.9688

%% file: c_p.txt
-0.9712

%% file: gr_ev.txt
-0.6641

%% file: smf_ev.txt
-0.6426

%% file: mf_ev.txt
-0.6428

%% file: src_ev.txt
-0.6528

%% file: c_ev.txt
-0.6503

%% file: gr_mse.txt
0.6760

%% file: smf_mse.txt
0.6552

%% file: mf_mse.txt
0.6554

%% file: src_mse.txt
0.6655

%% file: c_mse.txt
0.6636

%% file: gr_avg.txt
0.8165

%% file: smf_avg.txt
0.8113

%% file: mf_avg.txt
0.8076

%% file: src_avg.txt
0.8122

%% file: c_avg.txt
0.8085